\documentclass[11pt]{article}
\usepackage[utf8]{inputenc}
\usepackage[margin=1in]{geometry}

\usepackage{footmisc}

\usepackage[T1]{fontenc}    
\usepackage[bookmarks, colorlinks=true, plainpages = false, citecolor = blue,linkcolor=red,urlcolor = blue, filecolor = blue]{hyperref}
\usepackage{url}            
\usepackage{booktabs}       
\usepackage{amsfonts}       
\usepackage{nicefrac}       
\usepackage{microtype}      
\usepackage{xcolor}         
\usepackage{wrapfig}
\usepackage[toc,page]{appendix}

\usepackage{amsmath,amsthm,amssymb} 
\usepackage{bm,verbatim,dsfont,mathtools}
\usepackage{color,graphicx,appendix}
\usepackage{etoolbox}
\usepackage{tikz}
\usepackage{xr,xspace}
\usepackage{todonotes}
\usepackage{paralist}
\usepackage{caption,subcaption,soul}
\usepackage[ruled,vlined,linesnumbered]{algorithm2e}
\usepackage{appendix}

\usepackage{pgfplots, alphalph, subcaption}
\usepgfplotslibrary{groupplots}
\graphicspath{{figures/}}

\makeatletter

\newtheorem{theorem}{Theorem}
\newtheorem{lemma}{Lemma}
\newtheorem{proposition}{Proposition}
\newtheorem{corollary}{Corollary}
\theoremstyle{definition}
\newtheorem{definition}{Definition}

\newtheorem{remark}{Remark}

\newtheorem{example}{Example}

\usepackage{xspace,prettyref}

\newcommand{\reals}{{\mathbb{R}}}

\newcommand{\naturals}{{\mathbb{N}}}

\newcommand{\identity}{\mathbf I}

\newcommand{\diff}{{\rm d}}
\newcommand{\red}{\color{red}}
\newcommand{\blue}{\color{blue}}
\newcommand{\nbo}[1]{{\sf\color{orange}[#1]}}
\newcommand{\nb}[1]{{\sf\blue[#1]}}
\newcommand{\ls}[1]{\nbo{LS: #1}} 
\newcommand{\nbr}[1]{{\sf\red[#1]}}
\newcommand{\Expect}{\mathbb{E}}
\newcommand{\expect}[1]{\mathbb{E}\left[ #1 \right]}

\newcommand{\expects}[2]{\mathbb{E}_{#2}\left[ #1 \right]}

\newcommand{\Prob}{\mathbb{P}}

\newcommand{\prob}[1]{ \mathbb{P}\left\{ #1 \right\} }

\newcommand{\ie}{i.e.\xspace}

\newrefformat{eq}{(\ref{#1})}
\newrefformat{chap}{Chapter~\ref{#1}}
\newrefformat{sec}{Section~\ref{#1}}
\newrefformat{alg}{Algorithm~\ref{#1}}
\newrefformat{fig}{Fig.~\ref{#1}}
\newrefformat{tab}{Table~\ref{#1}}
\newrefformat{rmk}{Remark~\ref{#1}}
\newrefformat{clm}{Claim~\ref{#1}}
\newrefformat{def}{Definition~\ref{#1}}
\newrefformat{cor}{Corollary~\ref{#1}}
\newrefformat{lmm}{Lemma~\ref{#1}}
\newrefformat{prop}{Proposition~\ref{#1}}
\newrefformat{app}{Appendix~\ref{#1}}
\newrefformat{hyp}{Hypothesis~\ref{#1}}
\newrefformat{thm}{Theorem~\ref{#1}}
\newrefformat{ass}{Assumption~\ref{#1}}

\newcommand{\pth}[1]{\left( #1 \right)}
\newcommand{\qth}[1]{\left[ #1 \right]}
\newcommand{\sth}[1]{\left\{ #1 \right\}}

\newcommand{\abth}[1]{\left | #1 \right |}

\newcommand{\norm}[1]{\left\|{#1} \right\|_2}
\newcommand{\Norm}[1]{\|{#1} \|_2}

\newcommand{\lnorm}[2]{\left\|{#1} \right\|_{{#2}}}

\newcommand{\Fnorm}[1]{\lnorm{#1}{\rm F}}
\newcommand{\fnorm}[1]{\|#1\|_{\rm F}}
\newcommand{\iprod}[2]{\left \langle #1, #2 \right\rangle}

\newcommand{\indc}[1]{{\mathbf{1}_{\left\{{#1}\right\}}}}

\newcommand{\diag}[1]{\mathsf{diag} \left\{ {#1} \right\} }

\newcommand{\calB}{{\mathcal{B}}}

\newcommand{\calD}{{\mathcal{D}}}

\newcommand{\calG}{{\mathcal{G}}}
\newcommand{\calH}{{\mathcal{H}}}
\newcommand{\calI}{{\mathcal{I}}}

\newcommand{\calL}{{\mathcal{L}}}

\newcommand{\calN}{{\mathcal{N}}}

\newcommand{\calP}{{\mathcal{P}}}

\newcommand{\calR}{{\mathcal{R}}}
\newcommand{\calS}{{\mathcal{S}}}
\newcommand{\calT}{{\mathcal{T}}}

\newcommand{\calV}{{\mathcal{V}}}

\newcommand{\calX}{{\mathcal{X}}}
\newcommand{\calY}{{\mathcal{Y}}}

\newcommand{\Tr}{\mathsf{Tr}}

\renewcommand{\hat}{\widehat}
\renewcommand{\tilde}{\widetilde}

\newcommand{\FG}{\mathsf{FG}}
\newcommand{\FedProx}{{FedProx}\xspace}
\newcommand{\FedAvg}{{FedAvg}\xspace}
\newcommand{\hnorm}[1]{\left\|{#1} \right\|_{\mathcal{H}}}
\newcommand{\hprod}[2]{\left \langle #1, #2 \right\rangle_{\mathcal{H}}}
\newcommand{\Hnorm}[1]{\|{#1} \|_{\mathcal{H}}}
\newcommand{\Hprod}[2]{\langle #1, #2 \rangle_{\mathcal{H}}}
\newcommand{\abs}[1]{\left| #1 \right|}
\newcommand{\opnorm}[1]{\lnorm{#1}{\rm op}}

\newcommand{\tha}{{\bar \theta}}
\newcommand{\bx}{{\bf x}}
\newcommand{\fa}{{\bar f}}

\newenvironment{FedLS}[1][htb]
  {
   \begin{algorithm}[#1]%
  }{\end{algorithm}}

\newenvironment{FedLC}[1][htb]
  {
   \begin{algorithm}[#1]%
  }{\end{algorithm}}

\begin{document}

\title{
A Non-parametric View of FedAvg and FedProx: \\Beyond Stationary Points
}

\author{
Lili Su \\
Electrical and Computer Engineering \\
Northeastern University  
\and 
Jiaming Xu\\
The Fuqua School of Business\\
Duke University 
\and 
Pengkun Yang\thanks{Correspondence author. Email: \url{yangpengkun@tsinghua.edu.cn}.  J.~Xu is supported by the NSF Grants IIS-1838124, CCF-1850743, CCF-1856424,
and CCF-2144593.
P.~Yang is supported by the NSFC Grant 12101353.
} \\
Center for Statistical Science \\ 
Tsinghua University 
}

\date{\today}

\maketitle 

\begin{abstract}

Federated Learning (FL) is a promising decentralized learning framework and 
has great potentials in privacy preservation and in lowering the 
computation load at the cloud. Recent work showed that FedAvg and FedProx -- the two widely-adopted FL algorithms -- fail to reach the stationary points of the global optimization  objective
even for homogeneous linear regression problems. Further,  it is concerned that the common model learned might not generalize well locally at all in the presence of heterogeneity. 
 
In this paper, we analyze the convergence and statistical efficiency of FedAvg and FedProx, addressing the above two concerns. Our analysis is based on the standard non-parametric regression  in a reproducing kernel Hilbert space (RKHS), and allows for 
heterogeneous local data distributions and unbalanced local datasets. 
We prove that the estimation errors, measured in either the 
empirical
norm or the RKHS norm, decay
with a rate of $1/t$ in general and exponentially for finite-rank kernels. 
In certain heterogeneous settings, these upper bounds also imply that both FedAvg and FedProx achieve the optimal error rate. 
To further analytically quantify the impact of the heterogeneity  at each client,
we propose and characterize a novel notion-federation gain, defined as the reduction of the estimation error for a client to join the FL. We discover that when the data heterogeneity is moderate,  a client with limited local data can  benefit from a common model with a large federation gain.
Numerical experiments further corroborate our theoretical findings. 
\end{abstract}



\section{Introduction}
\label{sec: intro}

Federated Learning (FL) is a rapidly developing  decentralized learning framework in which a parameter server (PS) coordinates with a massive collection of end devices 
in executing machine learning tasks 
\cite{konevcny2016federatedb, konevcny2016federateda,mcmahan2017communication,kairouz2019advances}. 
In FL, instead of uploading data to the PS, the end devices work at the front line in processing their own local data and periodically report their local updates
to the PS. The PS then effectively aggregates those
updates to obtain a fine-grained model and broadcasts the fine-grained model to the end devices for further model updates. 
On the one hand, FL has 
great potentials in privacy-preservation and in lowering the computation load at the cloud, 
both of which are crucial for modern machine learning applications.  
On the other hand, the defining characters of FL, i.e., costly communication, massively-distributed system architectures, highly unbalanced and heterogeneous data across devices, make it extremely challenging to understand the theoretical foundations of  popular FL algorithms. 

\FedAvg and \FedProx are two widely-adopted FL algorithms \cite{mcmahan2017communication,li2018federated}. They center around minimizing a global objective function
$\ell(f) \triangleq \sum_{i=1}^M w_i \ell_i(f)$, where $\ell_i(f)$ is the local empirical risk of model $f$ evaluated at client $i$'s local data \cite{kairouz2019advances,mcmahan2017communication,li2018federated,karimireddy2020scaffold} and $w_i$ is the weight assigned to client $i$.
Specifically, in each round $t$, starting from $f_{t-1}$ each client $i$ computes its local model update $f_{i,t}$, which is then aggregated by the PS to produce $f_t$. 
To save communication, \FedAvg  
only aggregates the local updates 
every $s$-th step of local gradient descent, where $s\ge 1$; when $s=1$, \FedAvg reduces to the standard stochastic gradient descent (SGD) algorithm. 
\FedProx is a proximal-variant of \FedAvg where the local gradient descent is replaced by a proximal operator. 

Despite ample recent effort and some progress, 
the convergence and statistical efficiency of these two FL algorithms
remain elusive \cite{kairouz2019advances,pathak2020fedsplit}. In particular, existing attempts 
often impose either impractical or restricted assumptions such as balanced local data \cite{li2019convergence}, bounded gradients dissimilarity (i.e., $\nabla \ell_i \approx \nabla \ell$ for all $i$)  \cite{li2018federated,karimireddy2020scaffold,stich2018local,zinkevich2010parallelized}, and fresh data \cite{kairouz2019advances}, and mostly ignore  the impact of the model dimension (see Section \ref{subsec: related work} for detailed discussions). 
More concerningly, recent work \cite{pathak2020fedsplit,karimireddy2020scaffold,zhao2018federated} showed,
both experimentally and theoretically, that both  \FedAvg and \FedProx fail to reach the stationary point of $\ell(f)$ 
even for the simple homogeneous linear regression  problems. 
This observation is also illustrated in Fig.\,\ref{fig:f1}, wherein we plot the trajectories of the gradient magnitudes $\norm{\nabla \ell(f_t)}$ versus the communication rounds $t$ under \FedProx and \FedAvg with aggregation period $s$ being $1, 5, 10$, respectively. 
While the gradient magnitude of \FedAvg with $s=1$
quickly drops to $0$, the gradient magnitudes under 
\FedProx and \FedAvg with $s=5, 10$ 
stay well above $0.$ 
Does the failure of reaching stationary points lead to unsuccessful learning? 
We plot the evolution of the estimation error illustrated in Fig.\,\ref{fig:f2}.
Surprisingly, both \FedAvg with $s= 5,10$ and \FedProx quickly converge to almost the same estimation error as \FedAvg with $s= 1$ (i.e.\ the standard SGD). 
Moreover, the convergence time of \FedAvg with $s=5, 10$ shrinks roughly by a factor of $s$
compared to $s=1$, indicating that \FedAvg enjoys significant saving in communication cost. 
\emph{Why \FedAvg and \FedProx can achieve low estimation errors despite the failure of reaching stationary points?} 
The current paper aims to demystify this paradox. 

\begin{figure}
  \begin{subfigure}[b]{0.45\textwidth}
    \includegraphics[width=\textwidth]{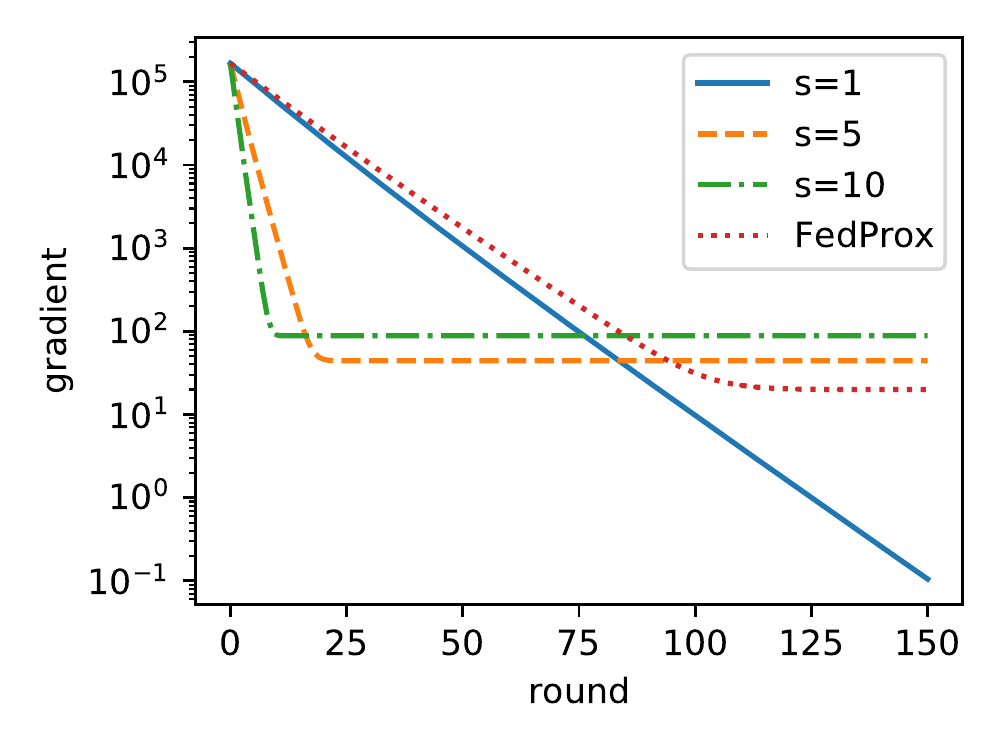}
    \caption{Plots of the gradient magnitudes versus the communication rounds}
    \label{fig:f1}
  \end{subfigure}
    \hfill
  \begin{subfigure}[b]{0.45\textwidth}
    \includegraphics[width=\textwidth]{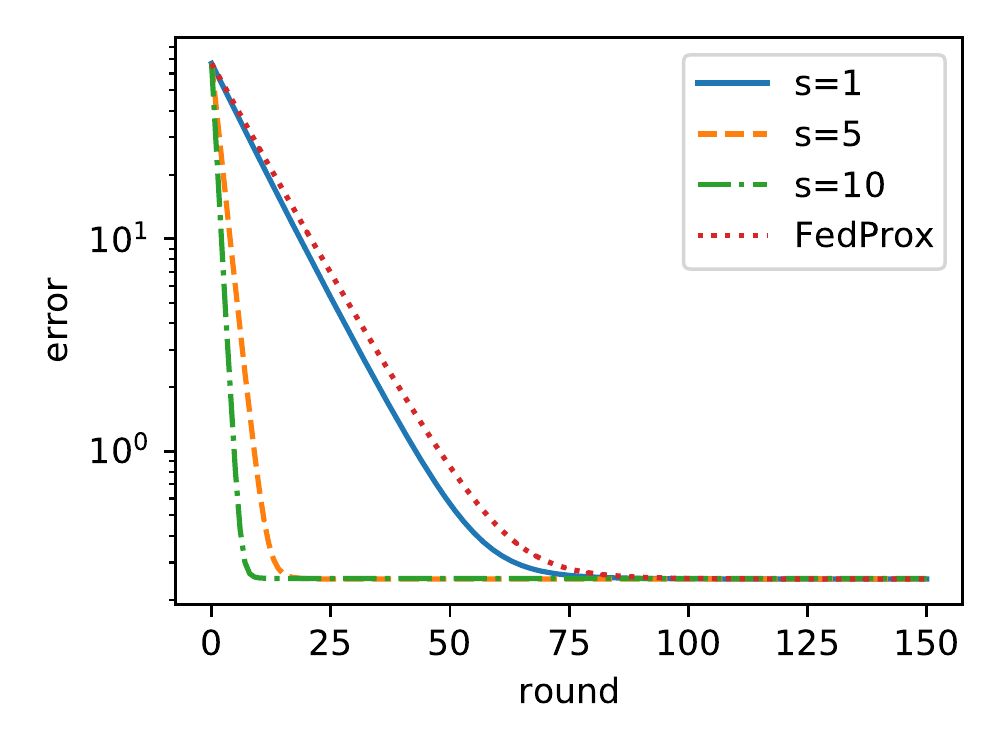}
    \caption{Plots of the estimation errors versus the communication rounds}
    \label{fig:f2}
  \end{subfigure}
\caption{Plots of linear regression under \FedProx and \FedAvg.
Experiment specifications: 25 clients, covariate dimension is 100, local sample size is 500, and observation noise follows $\calN(0, 0.25I)$. Detailed specifications can be found in Section \ref{subsec: stationary points v.s. prediction errors}. 
}
\label{fig: figure group 1}
\end{figure}

Our study is further motivated by the 
concern on the lack of model personalization. 
Note that under both \FedAvg and \FedProx, a common model is trained but is used to serve all the participating clients  without further tailoring to their local datasets. 
The tension between such standardized model and the data heterogeneity across clients leads to ever-increasing concern on the generalization performance of the common FL model at  different 
clients.   
In fact, on highly skewed heterogeneous data, evidence has been found suggesting that 
a common model could be problematic \cite{zhao2018federated,fallah2020personalized,deng2020adaptive,dinh2020personalized}. 
\emph{Under what scenarios can a client benefit from a common model in the presence of heterogeneity?}
The current paper seeks to address this question by quantifying the benefits and the impact of heterogeneity  via a novel notion -- {\em federation gain}. 


\paragraph*{Contributions} 
In this paper, we analyze the convergence and statistical efficiency of
\FedAvg and \FedProx by combining the optimization and  statistical perspectives.
Specifically, we assume that each client $i$ has $n_i$ local data points $\{x_{ij}, y_{ij}\}_{j=1}^{n_i}$ such that 
$
y_{ij} = f_i^*\left( x_{ij} \right) + \xi_{ij},$
where $f^*_i$ 
is the true model and $\xi_{ij}$ is the noise.
 We allow 
$n_i$,  $x_{ij}$, $f_i^*$, and $\xi_{ij}$
to vary across different clients $i$, 
capturing the unbalanced data partition, 
covariate heterogeneity, and model heterogeneity, 
which are three most important types of heterogeneity 
survey~\cite{kairouz2019advances}. 
We base our analysis on the standard non-parametric
regression  setup and assume that $f_i^*$ belongs to a
reproducing kernel Hilbert space (RKHS) $\calH$~\cite{wainwright2019high}.

We first show in Section \ref{sec:pred_error_emp} 
that the existence of heterogeneity does not prevent the convergence of $f_t$ to a good common model $f$ under \FedAvg and \FedProx. 
Specifically, we show in Theorem \ref{thm:pred_error_emp} that with a proper early stopping rule, the estimation error decays with a rate of $1/t$. 
This further implies that: (i) in the presence of only unbalanced data partition and covariate heterogeneity where $f_i^* \equiv f^*$, 
 $f_t$ converges to $f^*$; (ii)
 with additional model heterogeneity, $f_t$ approaches a common $f$ that balances the model discrepancy up to a
residual estimation error.
For finite-rank kernel matrices, we  further improve the convergence rate to be exponential  without early stopping in \prettyref{thm: exponential}. 
High probability bounds 
are derived in \prettyref{thm: Light-tailed noises} for both light-tailed and heavy-tailed noises.

We show in Section \ref{sec:theta_err} that the finite-rankness of the kernel  also enables us to derive an explicit expression of the common $\bar{f}$ that perfectly balances out the heterogeneity across clients. In fact, in~\prettyref{thm:conv_theta} we  establish
the convergence in RKHS norm - a strictly stronger notion of convergence. In particular, we show that the estimation error $\|f_t-\bar{f}\|_{\calH}$ decays exponentially fast to $O(\sqrt{d/N})$ for $N=\sum_i n_i$,
provided that the sample covariance matrix is well-conditioned. This error rate coincides with the minimax-optimal rate in the centralized setting. We further present two exemplary settings where the well-conditionedness assumption is shown to hold with high probability.

Moreover, we bound the difference $\|\bar f -f_j^*\|_{\calH}$, showing that when the model heterogeneity is moderate, a client with
limited local data can still benefit from a common model. To formally study the benefits of joining FL, in Section \ref{subsec: Characterization of Federation gains} we propose and characterize the \emph{federation gain}, defined as the reduction of the estimation error for a client to join the FL. 
We establish a threshold on the heterogeneity in terms of model dimensions and local data sizes under which the federation gain exceeds one. 
Our characterization of federation gain serves as a guidance in encouraging end devices owners to make their participation decisions.

Finally, using numerical experiments, we corroborate our theoretical findings. 
Specifically, 
in Section \ref{subsec: stationary points v.s. prediction errors} we demonstrate that both \FedAvg and \FedProx can achieve low estimation errors despite the failure of reaching stationary points. The same phenomenon is found to still persist when minibatches 
are used in local updates. In Sections \ref{subsec: experiments gains} and \ref{app: experiments: data heterogeneity}, we adapt the experiment setup to allow for unbalanced local data partition, 
covariate heterogeneity, and model heterogeneity. 
For both \FedAvg and \FedProx, we empirically observe that 
the federation gains are large when a client has a small local data size and the data heterogeneity is moderate, matching our theoretical predictions. In Section \ref{app: experiments: general regression}, we fit nonlinear models
and confirm that 
both FedAvg and FedProx can continue to attain nearly optimal estimation rates.

\subsection{Related work}
\label{subsec: related work}
\paragraph*{On convergence of \FedAvg and \FedProx}
\FedAvg has emerged as the algorithm of choice for FL \cite{kairouz2019advances,karimireddy2020scaffold}. 
Both empirical successes and failures of convergence have been reported \cite{mcmahan2017communication, li2018federated, karimireddy2020scaffold}, 
however, the theoretic characterization of its convergence (for general $s$) turns out to be notoriously difficult.   
In the absence of data heterogeneity, convergence is shown in 
\cite{zinkevich2010parallelized,stich2018local} under the name local SGD.  In particular, \cite{zinkevich2010parallelized} proves the asymptotic convergence. 
Convergence in the non-asymptotic regime is derived in \cite{stich2018local} under assumptions of strong convexity and bounded gradients.  
The proof techniques of \cite{zinkevich2010parallelized,stich2018local} are adapted to data heterogeneity setting 
by postulating the variances of gradients are bounded or the dissimilarities of gradients/Hessian are 
uniformly bounded~\cite{karimireddy2020scaffold}.
Even stronger assumptions are adopted for the convergence proof of \FedProx \cite{li2018federated}. 
Most of these results are derived in the context of optimization and focus on the training errors only.  Other work assumes fresh data in each update for technical convenience \cite{kairouz2019advances}. Both the randomness in the design matrix, which is harder to handle, and the impacts of the covariate dimension are mostly neglected. In particular, when the randomness in the design matrix is taken into account, ensuring uniformly bounded dissimilarity requires the local data size to be much larger than the model dimension --  excluding their applicability to locally data scarce applications such as Internet of Things and mobile healthcare.

\paragraph*{Personalization} In the context of Model Agnostic Meta Learning (MAML), personalized Federated Learning is investigated both experimentally \cite{chen2018federated,jiang2019improving} and theoretically 
\cite{fallah2020personalized,lin2020collaborative}. MAML-type personalized FL 
finds a {\em shared initial  model} that a participating device can quickly get personalized by running a few updates on its local data. 
Adaptive Personalized Federated Learning (APFL) is proposed in \cite{deng2020adaptive} under which each end device trains its local model while contributing to the global model. A personalized model is then learned as a mixture of optimal local and global models. Other personalization techniques include model division, contextualization, and multi-task learning. Due to space limitation, readers are referred to \cite{kairouz2019advances} for details. 
In this paper, we 
show that without introducing additional personalization techniques, an end device can still benefit from joining FL under certain mild conditions.

\section{Problem Formulation}
\label{sec: problem formulation}
\paragraph{System model.}
A federated learning system consists of a parameter server (PS) and $M$ clients. 
Each client $i\in \{1, \cdots, M\} \triangleq [M]$ locally keeps its personal data $\calS_i = \sth{\pth{x_{ij}, y_{ij}}}_{j=1}^{n_i}$,  where $n_i = \abth{\calS_i}$ is referred to as local data volume of client $i$. 
Let $N\triangleq \sum_{i=1}^M n_i$. 
It is possible that $n_i\not=n_j$ for some $i\not=j$, i.e., the local data volume at different clients could be highly unbalanced. The magnitude of $n_i$ varies with different real-world applications: when $\calS_i$ are records of recently browsed websites, $n_i$ is typically moderate;  
when $\calS_i$ are records of recent places visited by walk in pandemic, the volume of $\calS_i$ is low. Observing this, in this work, we consider a wide range of $n_i$ which covers both the small and moderate $n_i$ regions as special cases.  

\paragraph{Data heterogeneity.}
We consider both covariate heterogeneity (a.k.a.\ covariate shift) and response heterogeneity (a.k.a.\ concept shift) \cite{kairouz2019advances}.  
Formally, at each client $i$, 
$$
y_{ij} = f_i^*\left( x_{ij} \right) + \xi_{ij}, \quad 1 \le j \le n_i,
$$
where $f_i^*$ 
is the underlying mechanism governing the true responses, $x_{ij}\in \calX$ is the covariate, and $\xi_{ij}$ is the observation noise. 
We impose the mild assumptions that 
$\xi_{i1}, \cdots, \xi_{in_i}$ are independent yet possibly non-identically distributed, zero-mean,
and have variance up to $\sigma^2.$ 

\paragraph{Non-parametric regression.}
We base our analysis on the standard non-parametric
regression setup and assume that $f_i^*$ belongs to a 
reproducing kernel Hilbert spaces (RKHS)  $\calH$ with a defining positive semidefinite kernel function
$k: \calX \times \calX \to \reals$.
For completeness, we present the relevant fundamentals of RKHS 
(see~\cite[Chapter 12]{wainwright2019high} for an
in-depth exposition). 
Let $\iprod{\cdot}{\cdot}_{\calH}$ denote the inner product of the RKHS $\calH$. 
At any $x\in \calX$, $k(\cdot, x)$ acts as the representer of evaluation at $x$, i.e., 
\begin{align}
\iprod{f}{k(\cdot,x)}_{\calH} = f(x), \qquad \forall f\in \calH. \label{eq:kernel_eval}
\end{align}
Let $\|g\|_{\calH} = \sqrt{\iprod{g}{g}_{\calH}}$ denote the norm of function $g$ in 
$\calH$. 
For a given distribution $\mathbb{P}$ on $\calX$,   
let $\|g\|_2 = \pth{\int_{\calX}g(x)^2d\mathbb{P}(x)}^{1/2}$ denote the norm in $L^2(\mathbb{P})$. 
In this paper, we take the following minimal assumptions that are common in literature \cite{wainwright2019high}.   
We assume that $\calX$ is compact, $k$ is continuous, $\sup_{x \in \calX} k(x,x) <\infty$, 
and that $\int_{\calX\times \calX} k^2(x, z)d\mathbb{P}(x)d\mathbb{P}(z) < \infty$.  
Mercer's theorem shows that
such kernel $k$ admits an expansion 
\begin{equation}
\label{eq:k-expand}
k(x, z) = \sum_{\ell=1}^{\infty} \mu_{\ell}\varphi_{\ell}(x)\varphi_{\ell}(z),
\end{equation}
where $\{\varphi_\ell\}$ forms an orthonormal basis   of $L^2(\mathbb{P})$, and $\{\mu_\ell\}$ are the non-negative eigenvalues. 
Notably, $\iprod{\varphi_{\ell}}{\varphi_{\ell^{\prime}}}_{\calH} =0$ for $\ell\not=\ell^{\prime}$ and $\iprod{\varphi_{\ell}}{\varphi_{\ell}}_{\calH} =\frac{1}{\mu_{\ell}}$ for all $\ell$ such that $\mu_{\ell}\not=0$ \cite[Corollary 12.26]{wainwright2019high}.
Define the feature mapping $\phi: \calX \to \ell^2(\naturals)$ as 
$
\phi(x)=\left[\sqrt{\mu_1} \varphi_1(x), \sqrt{\mu_2} \varphi_2(x), \ldots \right],
$
where $\ell^2(\naturals)$ denotes the space of square-summable sequences.
Then for any $f \in \calH$ with $f(x)=\sum_{\ell=1}^{\infty} \beta_{\ell}\varphi_{\ell}(x)$ such that $\sum_{\ell=1}^{\infty}\frac{\beta_{\ell}^2}{\mu_{\ell}}<\infty$\footnote{With a little abuse of notation,  $\sum_{\ell=1}^{\infty}\frac{\beta_{\ell}^2}{\mu_{\ell}}$ sums over all $\ell$ such that $\mu_{\ell}>0$.},
we have $f(x) = \sum_{\ell=1}^{\infty}\theta_{\ell}\phi_{\ell}(x)$, where  $\theta_{\ell} = \frac{\beta_{\ell}}{\sqrt{\mu_{\ell}}}$. Hence,  
\begin{align}
\|f\|_2^2= \sum_{\ell=1}^{\infty} \beta_{\ell}^2, ~~~ \text{and} ~~~ \|f\|^2_{\calH} = \sum_{\ell=1}^{\infty} \theta_{\ell}^2.  \label{eq:norm_def}
\end{align}

The above non-parametric setting can be used to approximate more sophisticated settings. 
In particular, it is applicable to random kernels by using the corresponding eigenvalues and
thus covers the neural tagent kernels (NTKs) to approximate the NNs in certain regimes. For instance, the NTK for two-layer NNs is
$k(x,y)=\expects{\sigma'(w^{\top}x)\sigma'(w^{\top}y)}{w\sim \calN(0,I)}$,
where $\sigma$ is the activation function. 


\paragraph{Additional notation} 
Let $\bx_i=(x_{i1},\dots,x_{in_i})\in\calX^{n_i}$ denote the covariate of the local data at client $i$; all data covariate $\bx = (\bx_1,\dots,\bx_M)\in\calX^N$. Similarly, let $y_i\in \reals^{n_i}$ and $y\in \reals^N$ be the vectors that stack the responses of the local data at client $i$ and the total data, respectively.  
For $a,b\in\reals^d$, let $a\cdot b\triangleq \sum_{i=1}^d a_ib_i$. 
Given a multivariate function $f:\calX\mapsto \reals^d$, we use $f_j$ to denote the $j$-th components of $f$;
for $a\in\reals^d$, define $a\cdot f:\calX\mapsto \reals$ as $(a\cdot f) (x)=a\cdot f(x)$; 
for $A\in\reals^{n\times d}$, define $A f:\calX\mapsto \reals^n$ as $(A f)_i(x)\triangleq \sum_{j=1}^d A_{ij}f_j(x)$ for $i\in[n]$.
For $x\in\calX$, let $k_x\triangleq k(\cdot,x):\calX\mapsto \reals$; for $\bx=(x_1,\dots,x_n)\in\calX^n$, let $k_\bx\triangleq (k_{x_1},\dots,k_{x_n}):\calX\mapsto \reals^n$, and $K_\bx$ be the normalized Gram matrix of size $n\times n$ with $(K_\bx)_{ij}=\frac{1}{n}k(x_i,x_j)$.
Given a mapping $f:\calX\mapsto\calY$ and $\bx=(x_1,\dots,x_n)\in\calX^n$, let $f(\bx)\triangleq (f(x_1),\dots,f(x_n))$; in particular, when $\calY=\reals^d$, let $f(\bx)$ be a matrix of size $n\times d$ that stacks $f(x_i)$ in rows.
For an operator $\calL:\calH\mapsto \calH$ and $f=(f_1,\dots,f_n)\in\calH^n$, let $\calL f \triangleq (\calL f_1,\dots,\calL f_n)$.
Let $\norm{v}$ and $\norm{V}$ denote the $\ell^2$ norm of a vector $v$
and the spectral norm of matrix $V$, respectively. 
The operator norm is denoted by $\opnorm{\cdot}$.
For a positive definite matrix $A$, let $A^{1/2}$ denote the unique square root of $A$. 
Throughout this paper, we use $c, c_1, ...$ to denote absolute constants. For ease of exposition, the specific values of these absolute constants might vary across different concrete contexts in this paper. 
\section{\FedAvg and \FedProx}
\label{sec: algorithms}
%
\FedAvg can be viewed as a communication-light implementation of the standard 
SGD. 
Different from the standard SGD, wherein the updates at different clients are aggregated right after {\em every} local step, 
in \FedAvg the local updates are only aggregated after {\em every $s$-th} local step, where $s\ge 1$ is an algorithm parameter.  
\FedProx is a distributed proximal algorithm wherein a round-varying proximal term is introduced to control the deviation of the local updates 
from the most recent global model. 


Recall from Section \ref{sec: intro} that $\ell_i(f) = \frac{1}{2n_i} \sum_{j=1}^{n_i} \pth{f(x_{ij}) - y_{ij}}^2$ is the local empirical risk function for each $f\in \calH$.  
Let $f_t$ denote the global model at the end of the $t$-th communication round, and let $f_0$ denote the initial global model.  
At the beginning of each round $t\ge 1$, the PS 
broadcasts $f_{t-1}$ to each of the $M$ clients. 
At the end of round $t$, upon receiving the local updates $f_{i,t}$ from each client $i$, the PS 
updates the global model as 
\begin{align}
f_t= \sum_{i=1}^M w_i f_{i,t}\, , \label{eq:avg_step}
\end{align} 
where $w_i = \frac{n_i}{N}$ -- recalling that $N$ is the number of all the data tuples in the FL system.  
The local updates $f_{i,t}$ under \FedAvg and \FedProx are obtained as follows. 


\paragraph{\FedAvg}
From $f_{t-1}$ each client $i$ runs $s$ local gradient descent steps on $\ell_i(f)$, and reports its updated model to the PS.  
Concretely, we denote the mapping of one-step local gradient descent by
$\calG_i(f) = f-\eta \nabla \ell_i(f)$,
where $\eta>0$ is the stepsize. 
%
After $s$ local steps, 
the locally updated model at client $i$ is given by 
$$
f_{i,t} = \calG_i^s (f_{t-1}).
$$

\paragraph{\FedProx} 
%
From $f_{t-1}$, each client $i$ locally updates the model as 
\begin{align}
\label{eq:Fedprox local obj}
    f_{i,t} ~ = ~ \arg\min_{f\in \calH} ~ 
    \ell_i(f) 
    + \frac{1}{2\eta}\|f - f_{t-1}\|_{\calH}^2. 
\end{align}
Notably, $\eta>0$ controls the regularization and can be interpreted as a step size in the \FedProx: As $\eta$ increases, the penalty for moving away from $f_{t-1}$ decreases and hence the local update $f_{i,t}$ will be farther way from $f_{t-1}.$ 
In practice, the local optimization problem in \eqref{eq:Fedprox local obj} might not be 
solved exactly in each round. We would like to study the impacts of inexactness of solving \eqref{eq:Fedprox local obj} in future work.

\section{Recursive Dynamics of \FedAvg and \FedProx}
\label{sec: error iterate}
In this section, we derive expressions for the 
recursive dynamics of $f_t$ in \eqref{eq:avg_step}
under \FedAvg and \FedProx, respectively. 
All missing proofs of this section can be found in Appendix \ref{app: error iterate}. 
We first introduce two local linear operators. 
Within iteration $t$ of \FedAvg, the one-step local gradient descent on client $i$ is given by an affine mapping 
\begin{equation}
\label{eq:Pi-local}
\calG_i(f_{t-1})=f_{t-1}-\frac{\eta}{n_i}\sum_{j=1}^{n_i}(f_{t-1}(x_{ij})-y_{ij})k_{x_{ij}}
= \calL_i f_{t-1}+\frac{\eta}{n_i}\sum_{j=1}^{n_i} y_{ij}k_{x_{ij}},
\end{equation}
where $\calL_i$ denotes the local operator
\[
\calL_i f ~\triangleq ~ 
f-\frac{\eta}{n_i}\sum_{j=1}^{n_i}f(x_{ij})k_{x_{ij}}.
\]
For \FedProx,  the global model dynamics involves the inverse of local operator $\tilde\calL_i$ where
\[
\tilde\calL_i f 
\triangleq f+\frac{\eta}{n_i}\sum_{j=1}^{n_i}f(x_{ij})k_{x_{ij}}.
\]

Recall that $w_i = \frac{n_i}{N}$. The following proposition characterizes the dynamics of $f_t$. 
\begin{proposition}
\label{prop:theta_recursion}
The global model $f_t$ satisfies the following recursion:
\begin{equation}
\label{eq:iteration-f}
f_{t}=\calL f_{t-1} +  y\cdot \Psi \, ,
\end{equation}
where $\Psi=(w_1\Psi_1,\dots,w_M\Psi_{M}):\calX\mapsto \reals^N$ 
and
\begin{align}
    \calL=
    \begin{cases}
    \sum_{i=1}^M w_i \calL_i^s  & \\
    \sum_{i=1}^M w_i \tilde\calL_i^{-1} &
    \end{cases} \label{eq:def_A}
    \quad
    \text{ and }
    \quad
    \Psi_{i}=
    \begin{cases}
    \frac{\eta}{n_i}\sum_{\tau=0}^{s-1} \calL_i^\tau k_{\bx_{i}}  & \text{ for \FedAvg} , \\
    \frac{\eta}{n_i} \tilde\calL_i^{-1}k_{\bx_{i}} &\text{ for \FedProx} .
    \end{cases}
\end{align}
\end{proposition}
\FedAvg with $s=1$ coincides with the standard distributed gradient descent, which naturally fuses a global model as $\sum_i w_i\calG_i$ effectively aggregates all local data. 
However, for $s>1$, analyzing the dynamics in Proposition \ref{prop:theta_recursion} directly is challenging as the local updates involve high-order operators $\calG_i^s$. 
This makes the global model fusion more difficult because $\sum_i w_i \calG_i^s$ aggregates local progress in a nontrivial manner and further drives $f_t$ away from the stationary points of the global objective function $\ell(f).$
Similar challenges also appear in \FedProx due to the inverse of $\tilde\calL_i$.



Fortunately, from Proposition \ref{prop:theta_recursion} we can derive compact expressions of 
the evolution of the in-sample prediction values under \FedAvg and \FedProx, 
which serve as the foundation for the convergence analysis in Section \ref{sec: main results}.
Under \FedAvg with $s=1$, 
it is well-known in the literature of \emph{kernel methods} 
\cite[Chapter 12]{Hastie2009}
(and also follows from \eqref{eq:Pi-local}) that
\begin{equation}
\label{eq:kernel-GD}
f_t(\bx) = (I-\eta K_\bx)f_{t-1}(\bx) + \eta K_\bx y,
\end{equation}
For $s>1$, there is no immediate extension of \eqref{eq:kernel-GD} to $s>1$ and to \FedProx. 
The key step in our derivation is 
a set of identities for $\calL f_{t-1} (\bx)$ and $\Psi(\bx)$, which are stated in the next lemma. Those identities are also used in our convergence proofs, and  could be of independent interest to a broader audience.

\begin{lemma}\label{lmm:block_magic}
For any $f\in\calH$, the following identities are true:
\begin{align*}
& \Psi = \frac{\eta}{N} P k_\bx,
\qquad \qquad ~~
\Psi(\bx)=\eta K_\bx P,
\\
& f(\bx)\cdot \Psi =f-\calL f,
\qquad 
\calL f(\bx)=(I-\eta K_\bx P)f(\bx),
\end{align*}
where 
$P \in \reals^{N\times N}$ is a block diagonal matrix whose $i$-th diagonal block of size $n_i\times n_i$ is 
\begin{align}
\label{eq: def: Pii}
    P_{ii} =
    \begin{cases}
    \sum_{\tau=0}^{s-1} [ I - \eta K_{\bx_i}]^\tau & \text{ for \FedAvg},\\
    [ I+ \eta K_{\bx_i}]^{-1} & \text{ for \FedProx}.
    \end{cases}
\end{align}
\end{lemma}
\begin{proposition}
\label{prop: prediction error iterate}
The prediction value satisfies the following recursion:
\begin{align}
\label{eq:pred_error_recursion_1}
f_{t}(\bx)= \left[I - \eta K_\bx P\right]f_{t-1}(\bx) + \eta K_\bx P y.
\end{align}
\end{proposition}
\begin{proof}
By \prettyref{prop:theta_recursion},
\[
f_{t}(x_{ij})=\calL f_{t-1}(x_{ij}) + \Psi(x_{ij})\cdot y.
\]
Consequently,
applying \prettyref{lmm:block_magic} yields that
\[
f_{t}(\bx)=\calL f_{t-1}(\bx) + \Psi(\bx)y
=\left[I - \eta K_\bx P\right]f_{t-1}(\bx) + \eta K_\bx P y.\qedhere
\]
\end{proof}

The dynamics of the model $f_t$ in \eqref{eq:iteration-f} and the corresponding in-sample prediction values $f_t(\bx)$ in \eqref{eq:pred_error_recursion_1} are both governed by linear time invariant (LTI) systems with $y\in \reals^N$ as the constant system input.
Those autoregressions converge if all eigenvalues of $\calL$ and $I-\eta K_\bx P$ are less than one in absolute value, and locations of the eigenvalues such as the distance to the unit circle
have important implications for the model evolution \cite{BD2009,BD2016}.
Although it is challenging to characterize the eigenvalues of $\calL$ due to the insufficiency of local data, system heterogeneity, and the involved aggregation of high-order or inverse operators, the eigenvalues of $I-\eta K_\bx P$ in the evolution of prediction values are more tractable.

Compared with 
the classical kernel gradient descent, here the crucial difference is the effect of matrix $P$, which arises from multiple local updates of \FedAvg and the proximal term in the local update of \FedProx. 
In particular, it is essential to characterize the spectrum of $K_{\bx}P$. 
When $P$ is positive definite, 
analagous to the normalized graph Laplacians (see e.g.~\cite[Section 3.2]{von2007tutorial}),
the eigenvalues of $K_\bx P $ coincide
with those of the symmetric matrix $ P^{1/2} K_\bx P^{1/2}$, and hence must be real and non-negative. 
It follows that the eigenvalues of $I-\eta K_\bx P$
are no more than $1$. 
Define 
\[
\gamma\triangleq \eta \max_{i\in[M]}\norm{K_{\bx_i}}.
\]
By the block diagonal structure of $P$, 
$\gamma<1$ guarantees that $P\succ 0$, and furthermore both $\calL$ and $I-\eta K_\bx P$ have non-negative eigenvalues only.

\begin{lemma}
\label{lmm:eigen-KP}
If $\gamma<1$, 
then 
all eigenvalues of $\calL$ and $I-\eta K_\bx P$ are within $[0,1]$.
\end{lemma}


Throughout this paper, we assume $\gamma<1$\footnote{For \FedProx,
our results continue to hold without any assumption on $\gamma.$ In particular, the matrix $P$ is always positive definite regardless of $\gamma$.
In a sense, \FedProx is more stable than \FedAvg. Yet, the conditioning of $P$ degrades with $\gamma$.}.
The local update and the global aggregation
are stable if $P$ is well-conditioned, e.g., $P=I$ for the gradient descent.
In general, we have the following upper bound on the condition number of $P$.
\begin{lemma}
\label{lmm:condition-P}
\begin{equation}
\label{eq:def_kappa}
\norm{P}\norm{P^{-1}}
\le 
\kappa  \triangleq 
\begin{cases}
 \frac{\gamma s}{1- (1-\gamma)^s} &  \text{for \FedAvg}, \\
 1+\gamma &  \text{for \FedProx}. 
\end{cases} 
\end{equation}
Moreover, we have
\begin{align}
\Lambda_{i} \in
\begin{cases}
\left[  \lambda_{i} s/\kappa , \lambda_i s \right] &  \text{for \FedAvg}, \\
\left[ \lambda_{i} /\kappa , \lambda_i \right] &  \text{for \FedProx}, 
 \end{cases} 
 \label{eq:eigen_comp}
\end{align}
where $\lambda_i \ge 0$ and $\Lambda_i \ge 0$
 are the $i$-th largest eigenvalue of $K_\bx$ and $K_\bx P$, respectively.
\end{lemma}

From \prettyref{lmm:condition-P} and the definition of $\gamma$, $\kappa$ -- the upper bound to the conditioning number of $P$ -- approaches 1 with properly chosen small learning rate $\eta$ and small number of local steps $s$ in \FedAvg. 
Larger $\eta$ and $s$ accelerate the optimization and reduce the communication rounds at the expense of worsening  the conditioning of $P$ and incurring a larger statistical error; this tradeoff will be quantified in Section \ref{sec: main results}. 


\begin{remark}
When $k$ is a neural tangent kernel (NTK) \cite{Du2018,Du2019}, the kernel matrix $K_\bx$ is positive definite provided that the input training data is non-parallel. Therefore, the series of the gradient descent \eqref{eq:kernel-GD} given by
\[
f_t(\bx) = (I-\eta K_\bx)^t f_0(\bx) + (I-(I-\eta K_\bx)^t)y
\] 
converge to $y$ and thus attain zero training error for a properly small learning rate $\eta$.
It immediately follows from \eqref{eq:pred_error_recursion_1} that both \FedAvg and \FedProx attain zero training error for NTKs. 
\end{remark}


\section{Convergence Results}
\label{sec: main results}
In this section we present our results on the convergence of \FedAvg and \FedProx in terms of both the global model $f_t$ and the model coefficients $\theta_t$ -- recalling that $f_t = \iprod{\phi}{\theta_t}$, where $\phi$ is the feature mapping. For ease of exposition, we state our results for \FedAvg and \FedProx in a unified and compact form with $s$  as one characterizing parameter. 
 Recall that $s$ is the algorithm parameter of \FedAvg only. To recover the formal statements and involved quantities for \FedProx,  
we only need to set $s=1.$

%
%
%

\begin{itemize}
    \item To study the convergence of $f_t$, we 
compare 
$f_t$ with any given function $f \in \calH$ at the {\em observed covariates}. In particular, we study the prediction error,
as measured in the (empirical) $L^2(\mathbb{P}_N)$ norm, that is
\begin{align}
\lnorm{f_t-f}{N}^2 \triangleq \frac{1}{N} \norm{f_{t}(\bx)- f(\bx) \,}^2 
= \frac{1}{N} \sum_{i=1}^M \sum_{j=1}^{n_i}\pth{f_t(x_{ij}) -f(x_{ij})}^2. \label{eq:pred_L2}
\end{align}
Note that the $L^2(\mathbb{P}_N)$ norm is a commonly adopted performance metric in regression (See e.g.~\cite[Sections 7.4 and 13.2]{wainwright2019high}). 
Different from the training error $(1/N)\norm{f_t(\bx)- y}^2 $,
 the prediction error under \prettyref{eq:pred_L2} is able to reflect the over-fitting phenomenon. Concretely, when an algorithm is over-fitting noises, the training error could approach 0 whereas the prediction error under \prettyref{eq:pred_L2} would stay large.  
\item When the RKHS  $\calH$ is of finite dimension, we further study the convergence of $f_t$ in the RKHS $\calH$ norm. This is equivalent to the convergence of the model coefficient $\theta_t$ in the $L^2$ norm in view of~\prettyref{eq:norm_def}.
One can readily check that the convergence of the model $f_t$ in the $\calH$ norm is 
stronger than that in the $L^2(\mathbb{P}_N)$ norm.\footnote{By the reproducing property of kernels (i.e., the identity~\prettyref{eq:kernel_eval}) and the Cauchy-Schwarz inequality, we have 
\begin{align}
\left(f_t(x)-f(x)\right)^2 = \iprod{f_t-f}{k_x}_{\calH}^2 \le \lnorm{f_t-f}{\calH}^2 \lnorm{k_x}{\calH}^2 = \lnorm{f_t-f}{\calH}^2 k(x,x), ~~~~ \text{for any } x \in \calX. 
\label{eq:conv_comp}
\end{align}
Since $\sup_{x\in \calX} k(x,x)<\infty$, 
the convergence of $f_t$ to $f$ in $\calH$ norm implies the convergence in $L^2(\mathbb{P}_N)$ norm.}
\end{itemize}

\subsection{Convergence of prediction error 
} \label{sec:pred_error_emp}
%
%
%
%

The following proposition bounds the expected prediction error in terms of the eigenvalues of $K_\bx P$, 
denoted as $\Lambda_{1} \ge \Lambda_{2} \ge \cdots \ge\Lambda_{N} \ge 0$ as per \prettyref{lmm:condition-P}. 
\begin{proposition}
\label{prop:pred_error_general}
For any $f \in \calH$, 
it holds that for all $t \ge 1$
\begin{align}
\expects{\lnorm{f_t-f}{N}^2}{\xi} 
\le 3 \kappa \left( \delta_1 (t) \lnorm{f_0-f}{\calH}^2
+   \delta_2 (t)  \sigma^2
+\frac{1}{N} \norm{\Delta_{f}}^2
\right), 
\label{eq:error_bound_general}
\end{align}
where 
\begin{align}
\delta_1 (t) &=  \frac{1}{s} \max_{1 \le i \le N} \left( 1 - \eta \Lambda_{i} \right)^{2t} \Lambda_{i} \le \frac{1}{2e\eta ts}, \label{eq:delta_1_bound}\\
\delta_2(t) & = \frac{1}{N} \sum_{i=1}^N \left( 1 - \left( 1- \eta \Lambda_{i} \right)^{t} \right)^2 \le \frac{1}{N} \sum_{i=1}^N \min \left\{1, \eta t \Lambda_{i} \right\},
\label{eq:delta_2_bound} \\
\Delta_f &=\left( f_1^*(\bx_1), f_2^*(\bx_2), \cdots, f_M^*(\bx_M)\right) - f(\bx). \label{eq:def-Delta-f}
\end{align}
\end{proposition}
%
The expectation in~\prettyref{prop:pred_error_general} is only taken over the observation noise $\xi$ which has zero mean and bounded variance. 
The above result nicely separates 
the impact of bias, variance, and heterogeneity on the error dynamics. 
\begin{itemize}
    \item In \prettyref{eq:error_bound_general}, the first term on the right hand side is of the order $\delta_1(t) \lnorm{f_0-f}{\calH}^2$ and is 
related to the bias in estimation. 
 As indicated in~\prettyref{eq:delta_1_bound}, $\delta_1(t)$ decreases to  
$0$ as iterations proceed. 
The upper bound of $\delta_1(t)$ in \eqref{eq:delta_1_bound}, which decreases at a rate $c/t$, for a constant $c$ independent of $N$ and the kernel function $k$.  When the kernel matrix $K_{\bx}$ is of rank $d$, the convergence rate can be improved to be $\exp(-c_d t)$ for a constant $c_d$ independent of $N$.



\item The second term on the right hand side of~\prettyref{eq:error_bound_general} is of the order $\delta_2(t) \sigma^2$ and characterizes the variance in estimation. 
Note that $\delta_2(t)$ is capped at 1 and is increasing in $t$. Specifically, it converges to $1$ as $ t \to \infty$, capturing the phenomenon of over-fitting to noises.

\item The third term on the right hand side of~\prettyref{eq:error_bound_general} is of the order $\norm{\Delta_{f}}^2/N$ 
and quantifies the impact of the heterogeneity with respect to $f$. 
In the presence of only unbalanced data partition and covariate heterogeneity, we have $f_i^*=f^*$ for all $i$ and naturally $\Delta_{f^*}=0$.
Somewhat surprisingly, even under additional model heterogeneity that $f_i^*\neq f_j^*$, with assumptions such as invertibility of $\calI-\calL$,  there exists a choice of $f$ under which  $\Delta_f=0$ (cf.~\prettyref{eq:fa-equality}).
\end{itemize}

\medskip
To 
prevent over-fitting,  i.e., to control $\delta_2(t)$,
we can terminate the algorithms at some time $T$ before they enter the over-fitting phase.  The stopping time $T$ needs to be carefully chosen to balance the bias and variance \cite{raskutti2014early}.  
Note that $\delta_1(t) \le \frac{1}{2e\eta ts}$. 
To further control 
$\delta_2(t)$, we need to introduce the empirical Rademacher  complexity~\cite{bartlett2005local} defined as 
\begin{align}
\calR (\epsilon) = \sqrt{\frac{1}{N} \sum_{i=1}^N \min\{ \lambda_i, \epsilon^2 \}}, \label{eq:R_K_def}
\end{align}
where $\lambda_1 \ge \lambda_2 \ge \cdots \ge \lambda_N \ge 0$ are the eigenvalues of kernel matrix $K_\bx$  as per \prettyref{lmm:condition-P}.  
Intuitively, $\calR(\epsilon)$ is a data-dependent 
complexity measure of the underlying RKHS  and decreases with faster eigenvalue decay and smoother kernels. Recall from~\prettyref{eq:eigen_comp} that
$\Lambda_{i} \le \lambda_i s$. Hence it follows from~\prettyref{eq:delta_2_bound} that 
$$
\delta_2(t) \le \eta t s \calR^2 \left( 1/\sqrt{\eta t s} \right).
$$
Therefore, we can set $T$ as follows: 
 \begin{align}
 T \triangleq \max \left\{t \in \naturals:  \calR \left( 1/\sqrt{\eta t s} \right)  \le  \frac{1}{\sqrt{2e}\sigma \eta ts}\right\}.
 \label{eq:def_early_stopping}
 \end{align}
That is, we choose $T$ to be the largest time index $t$ so
that roughly the bias $\frac{1}{\eta ts}$ dominates the variance $\eta t s \calR^2\left( 1/\sqrt{\eta t s} \right) \sigma^2.$ 


\medskip
With early-stopping, we can specialize the general convergence in~\prettyref{prop:pred_error_general} as follows. 


\begin{theorem}[With early-stopping]
\label{thm:pred_error_emp}
For any $f \in \calH$, 
it holds that for all $1 \le t \le T$,
$$
\expects{\lnorm{f_t-f}{N}^2}{\xi} 
\le \frac{3 \kappa }{2e\eta ts} \left( \lnorm{f_0-f}{\calH}^2 +1 \right) + 
\frac{3\kappa}{N} \norm{\Delta_{f}}^2. 
$$
\end{theorem}
%


Our result in~\prettyref{thm:pred_error_emp} shows that the average prediction error decays at a rate of $O(1/t)$ and eventually saturates at the heterogeneity term $\frac{3\kappa}{N}\norm{\Delta_f}^2$.  This encompasses as a special case the existing convergence result of the centralized gradient descent for non-parametric regression~\cite{raskutti2014early} wherein similar early stopping is adopted with the specification $s=1$ and $f_i^*=f^*$ for all $i\in [M]$. 
%
%
\prettyref{thm:pred_error_emp} also reassures the common folklore and confirms 
our empirical observation in~\prettyref{fig:f2} on \FedAvg.
Specifically, with multiple local steps $s$ up to a certain threshold,
the convergence rate increases proportionally to $s$ while the final convergence error stays almost the same, i.e., we can recoup the accuracy loss while
 enjoying the saving of the communication cost. 
We cannot set $s$ to be arbitrarily large because 
as $s$ gets larger, the prediction error increases by a factor of $\kappa $, which is an increasing function of $s.$

\begin{remark}[Convergence in $L^2(\mathbb{P})$ norm]
We can also establish a uniform bound to the RKHS norm of $f_t-f$ up to the early stopping time $T$, as stated in~\prettyref{lmm:bound_ft_H} in the appendix. 
Furthermore, when $\bx_i$'s are i.i.d., this allows us to apply the empirical process theory to extend the bounds of \prettyref{eq:pred_L2} to those of the prediction error evaluated
at the unseen data, i.e., $\expects{(f_t(x)-f(x))^2}{x \sim \mathbb{P}}$ 
(see e.g.~\cite{raskutti2014early} and~\cite[Chapter 14]{wainwright2019high}).
For many 
kernels including polynomials and Sobolev classes, this yields the centralized minimax-optimal estimation error rate  \cite{yang1999information,raskutti2012minimax}. 
\end{remark}

%
%

%
\prettyref{thm:pred_error_emp} bounds the prediction error in expectation. 
In practice, the distributional structures of $\xi$ vary across different applications. High-probability bounds on $\lnorm{f_t-f}{N}^2$
can be obtained accordingly. 
\begin{theorem}[High-probability bounds]
\label{thm: Light-tailed noises}
For any $f\in \calH$ and any $t<T$, let 
\begin{align}
\varepsilon_t = \prob{\lnorm{f_t-f}{N}^2
\ge \frac{3\kappa }{2e \eta ts}
\left( \lnorm{f_0-f}{\calH}^2 +3 \right)+ \frac{3\kappa}{N} \norm{\Delta_{f}}^2 }.
\label{eq:def_epsilon_t}
\end{align}
\begin{itemize}
\item (Sub-Gaussian noise): Suppose the coordinates of the noise vector $\xi$
are $N$ independent zero-mean and sub-Gaussian variables (with sub-Gaussian norm bounded by
$\sigma$). There exists a  universal constant $c>0$ such that 
\begin{align*}
\varepsilon_t \le \exp\left(-c N/(\sigma^2 \eta t s) \right).
\end{align*}
\item (Heavy-tailed noise): Suppose the coordinates of $\xi$
are $N$ independent random variables with  $\Expect[\xi_i]=0$, $\Expect[\xi_i^2] \le \sigma^2$, and $\Expect|\xi_i|^p\le M_p <\infty$ for $p\ge 4.$ There exists a constant $c_p$ that only depends on $p$
such that 
$$
\varepsilon_t  \le c_p M_p \left( \frac{  \eta ts }{ N\sigma^2 } \right)^{p/4} .
$$
\end{itemize}
\end{theorem}

\prettyref{thm: Light-tailed noises} shows that the failure probability decays to $0$ as the sample size $N$ tends to infinity; the decay rate is exponential for sub-Gaussian noise and polynomial for noise with bounded moment.

%
%

%
%
\medskip
In general, we cannot hope to get a convergence rate that is strictly better than  $O(1/t)$.  
This is because the minimum eigenvalue $\lambda_N$ of the kernel matrix
is {\em not} bounded away from $0$ and may converge to $0$ as $N$ diverges. 
Fortunately, when the kernel matrix $K_\bx$ has a finite rank $d$,  the convergence rate can be improved to be exponential. 
%
%
\begin{theorem}[Exponential convergence for finite-rank kernel matrix]
\label{thm: exponential}
Suppose that the kernel matrix $K_\bx$ has finite rank $d$. Then 
$$
\expects{\lnorm{f_t-f}{N}^2}{\xi} 
\le 3 \frac{ \kappa }{\eta s}  \lnorm{f_0-f}{\calH}^2\exp\left( - 2 \frac{\eta s}{\kappa} \lambda_d t \right) + 
3 \kappa \sigma^2 \frac{d}{N}  + \frac{3\kappa}{N} \norm{\Delta_{f}}^2, \quad \forall t. 
$$
\end{theorem}
Enabled by the finite-rankness of $K_\bx$, \prettyref{thm: exponential} can be deduced from~\prettyref{prop:pred_error_general} via deriving a tighter upper bound on $\delta_1(t)$. 
Finite-rankness also ensures that the variance term is upper bounded by $\sigma^2 d/N$ -- hence no early stopping is needed. The complete proof is deferred to~\prettyref{sec:proof_corollaries}.

%

\subsection{Convergence of model coefficients}
\label{sec:theta_err}

In this section, we show the convergence of model coefficient $\theta_t$, or equivalently, the convergence of $f_t$ in RKHS norm. 
As shown in~\prettyref{eq:conv_comp}, this notion of convergence is strictly stronger than the convergence of $f_t$ in $L^2(\mathbb{P}_N)$ norm. 
For tractability, we assume that the RKHS is $d$-dimensional,
or equivalently $\phi(x)$ is $d$-dimensional\footnote{This further implies that $K_{\bx}$ is of rank at most $d$. 
}.
This encompasses the popular random feature model which maps the input data
to a randomized feature space~\cite{rahimi2007random}. 

\begin{theorem}\label{thm:conv_theta}
Suppose that $\phi(x)$ is $d$-dimensional.
%
Then
\begin{align}
    \expects{\norm{\theta_t - \tha}^2}{\xi}
    \le \left( 1- \frac{s \eta \rho_N }{\kappa} \right)^{2t}
    \norm{\theta_0-\tha}^2+ \sigma^2 \frac{\kappa d}{N \rho_N} \label{eq:theta_bound_desired}, 
\end{align}
where $\tha$ is the model coefficient of
$
\fa = (\calI-\calL)^{-1}\pth{(f_1^*(\bx_1),\dots,f_M^*(\bx_M))\cdot \Psi},
$
and $\rho_N= \frac{\lambda_{\min}(\phi(\bx)^\top \phi(\bx))}{N}.$
Moreover, the distance between $\bar{\theta}$ and $\theta_j^*$ is upper bounded by
\begin{align}
\label{eq:model_bound}
\norm{\tha-\theta_j^*}
\le
\norm{\Delta_{f_j^*}}\sqrt{\frac{\kappa}{N\rho_N}}.
\end{align}
\end{theorem}

High probability bounds, similar to \prettyref{thm: Light-tailed noises} but for $\theta_t$, can be obtained. A few explanations of~\prettyref{thm:conv_theta} are given as below.
\begin{itemize}
    \item In view of \prettyref{lmm:block_magic} and the definition of $\fa$, it holds that 
	\begin{equation}
	\label{eq:fa-equality}
	\fa(\bx)\cdot \Psi = (\calI-\calL)\fa = (f_1^*(\bx_1),\dots,f_M^*(\bx_M))\cdot \Psi,
	\end{equation}
	and hence $\Delta_{\bar f} \cdot \Psi =0.$
This turns out to be sufficient to ensure that  the global model $\fa$ balances out the impact of covariate and model heterogeneity across all clients.
    
    \item From~\prettyref{eq:iteration-f}, we expect that $f_t$ converges to the limiting point $f_{\infty}=(\calI-\calL)^{-1} (y\cdot \Psi)$. While $f_{\infty}$ can be far from being the stationary points of the global objective function $\ell(f)$, it is always an unbiased estimator of $\fa$. 
    
    \item  Note that $\rho_N$ depends on $N$. 
    When $N$ is sufficiently large, which is often the case as $N$ is the {\bf total} number of data points collectively kept by all the $M$ clients, $\rho_N$ is lower bounded by some positive constant\footnote{Note that $\phi(\bx)^\top \phi(\bx) \in \reals^{d\times d} $ is the covariance matrix, which is different from  the kernel matrix $\phi(\bx)\phi(\bx)^\top\in \reals^{N\times N}$ whose minimum eigenvalue is $0$ when $N>d.$
    }. 
 An example can be found in the analysis of~\prettyref{cor:conv_theta}, where it is shown that $\rho_N>\frac{\alpha}{2}$ for a fixed constant $\alpha>0$
 and all sufficiently large $N$. 
\end{itemize}

 \prettyref{thm:conv_theta} casts two key messages, highlighted in italic font below.  

\vskip 0.5\baselineskip

\noindent{\em \ul{Statistical optimality:}} When $\rho_N$ is lower bounded by a constant independent of $N$, as $t \to \infty$, the estimation error in~\prettyref{eq:theta_bound_desired} converges to $O\pth{d/N}$,  
which coincides with the minimax-optimal rate for estimating an $d$-dimensional vector in the centralized setting. 
Thus, our results immediately imply that when $f_i^* = f_j^*$, even in the presence of covariate heterogeneity, \FedAvg and \FedProx can achieve statistical optimality by effectively fusing the multi-modal data collected by the clients. \\


\noindent{\em \ul{Benefits of Federated Learning:}}  
The impact of the model heterogeneity  is quantified in 
\prettyref{eq:model_bound}, which says that 
$\bar{\theta}$ will stay within a bounded distance to its true local model $\theta_j^*$. 
In particular, 
when $n_j\ll d$, though client $j$ cannot learn any meaningful model based on its local dataset, by joining FL it can learn a model which is a reasonable estimation of $\theta_j^*$ despite heterogeneity. 
We formally quantify the benefits of joining FL in depth in Section \ref{subsec: Characterization of Federation gains}. \\

Depending on the underlying statistical structures of $\phi(\bx)$, $\rho_N$ and $\Norm{\Delta_{f_j^*}}$ can be further quantified. To cast insights on the magnitudes on $\rho_N$ and $\Norm{\Delta_{f_j^*}}$, next we will present some results on a couple of specific settings. 

\subsubsection{Covariate heterogeneity with bounded second-moments}
\label{subsub: isotropic}
%
%

\begin{corollary}\label{cor:conv_theta}
Suppose that 
$\phi(\bx)$ is a $N \times d$ matrix whose rows  are 
independent sub-Gaussian 
with the second-moment matrix $\Sigma_{ij} =\Expect[\phi(x_{ij})\phi(x_{ij})^\top]$.
Assume that $\alpha I \preceq \Sigma_{ij} \preceq \beta I$
for some fixed constants $\alpha, \beta>0. $
There exist constants $c_1, c_2$ that only depend on $\alpha, \beta$ such that if 
$N \ge c_1 d$, then with probability at least $1-e^{-d}$,
\begin{align}
    \expects{\norm{\theta_t - \tha}^2}{\xi}
    \le \left( 1- \frac{s \eta }{2\kappa} \right)^{2t}
    \norm{\theta_0-\tha}^2+ \sigma^2 \frac{2\kappa d}{N\alpha} . \label{eq:conv_theta_random}
\end{align}
Moreover, 
with probability at least $1-e^{-N}$,
\begin{align}
\norm{\tha-\theta_j^*} \le c_2\Gamma \sqrt{\kappa} ,
\label{eq:model_theta_random}
\end{align}
where $\Gamma =\max_{i,j}\Hnorm{f_i^*-f_j^*} =\max_{i,j}\Norm{\theta_i^*-\theta_j^*}.$  
\end{corollary}

\prettyref{cor:conv_theta} follows from~\prettyref{thm:conv_theta},
by showing that with high probability over the randomness of the covariate $\phi(\bx)$, the matrix  $\phi(\bx)^{\top}\phi(\bx)$ is positive definite with $\rho_N\ge \alpha/2$ and moreover $\Norm{\Delta_{f_j^*}}\lesssim \Gamma \sqrt{N}$. 
%



\subsubsection{Covariate heterogeneity with distinct and singular covariance matrices}
\label{subsec: distinct covariance}
In this section, we consider distinct covariance matrices, and  relax the requirement on the positive-definiteness of $\phi(\bx)^{\top}\phi(\bx)$. In particular, we consider the interesting setting wherein the rows of $\phi(\bx)$ are drawn from possibly different subspaces of low dimensions. This instance captures a wide range of popular FL applications such as image classification wherein different clients collect different collections of images \cite{mcmahan2017communication} --   
some clients may only have images related to airplanes or automobiles while others have images related to cats or dogs. 

Suppose the local features on client $i$ lie in a subspace of dimension $r_i$.
Let $\{u_{i1},\dots,u_{ir_i}\}$ denote an orthonormal basis of that subspace.
The local features $\phi(\bx_i)$ can be decomposed as $\phi(\bx_i) = \sqrt{d/r_i}F_i U_i^\top$, where $U_i=[u_{i1},\dots,u_{ir_i}]\in \reals^{d\times r_i}$, and $F_i\in\reals^{ n_i \times r_i}$ consists of the normalized coefficients.
The scaling $\sqrt{d/r_i}$ serves as the normalization factor of the signal-to-noise ratio due to $\Fnorm{U_i}=r_i$.
Furthermore, suppose that the local subspace $U_i$'s are independent with $\expect{U_i U_i^\top} 
=\frac{r_i}{d} I_d$; for instance, the subspace is uniformly generated at random.
Despite the singularity of $\phi(\bx_i)$, we show that the statistical accuracy only depends on the conditioning within the local subspace, i.e., the conditioning of $F_i$.



\begin{corollary}\label{cor:subspace}
Suppose $\lambda_{\min}(F_i^\top F_i/n_i) \ge \alpha $ and $\Norm{F_i^\top F_i/n_i}\le \beta$ for $i=1,\dots, M$ for some $\alpha,\beta>0.$
There exist a universal constants $C$ such that if $N \ge C\nu d \log d$, where $\nu \triangleq \max_{i \in [M]} n_i/r_i$,  then with probability at least $1-1/d$:  
\begin{align}
    \Expect_\xi\left[\norm{\theta_t - \tha}^2\right]
    \le \left( 1- \frac{s \eta \alpha }{2\kappa} \right)^{2t}
    \norm{\theta_0-\tha}^2+ \sigma^2 \frac{2\kappa d}{N \alpha}, \label{eq:conv_theta_random_2}
\end{align}
and
\begin{align}
\norm{\tha-\theta_j^*} \le 
\Gamma 
\sqrt{ \frac{2\kappa \beta \nu M d}{\alpha N}}.
\label{eq:model_theta_random_2}
\end{align}
\end{corollary}


The requirement on $\lambda_{\min}(F_i^\top F_i/n_i)$ is imposed to ensure that the local data at client $i$ 
contains strong enough signal about $\theta_i^*$
on every dimension of the subspace given by $U_i.$ To appreciate the intuition behind this requirement, it is instructive to consider the following two examples:
\begin{example}[Orthogonal local dataset]
\label{ex: Orthogonal local dataset}
Suppose that the rows of $\phi(\bx_i)$ are orthogonal to each other and each of which has Euclidean norm $\sqrt{d}$. In this case, we have $r_i=n_i$ and $F_i= \sqrt{r_i} I_{r_i}$. Therefore, $\alpha=\beta=\nu= 1.$ Then~\prettyref{cor:subspace} implies that as long as $N \ge Cd \log d$,
    $\theta_t$ converges exponentially fast to $\tha$ up to the optimal mean-squared error rate $d/N.$
\end{example}

\begin{example}[Gaussian local dataset]
\label{ex: Gaussian local dataset}
    Suppose $\phi (\bx_i) = \sqrt{d/r_i}F_i U_i^\top$, where the rows of $F_i$ are i.i.d.\ $\calN(0, I_{r_i})$. 
    In this case, by Gaussian concentration inequality~\cite[Theorem 5.39]{vershynin2010nonasym},
    with high probability 
    $1-\delta \le \alpha \le \beta \le 1+\delta $
    for some small constant $\delta>0$, provided that $n_i\ge C \max\{r_i,\log M\}$ for some sufficiently large constant $C.$
    Then~\prettyref{cor:subspace} implies that if further $N \ge C \nu d \log d$,
    $\theta_t$ converges exponentially fast to $\tha$ up to the optimal mean-squared error rate $d/N.$  

\end{example}

Note that we pay an extra factor of $\nu$ in the sample complexity in~\prettyref{cor:subspace}. This is necessary in general. To see this, consider the extreme case where $r_i=1$ and $n_i=n$, \ie, all local data at client $i$ lie on a straight line in $\reals^d$. Then by the standard coupon collector's problem, we need $M \ge d \log d$ in order to sample all the $d$ basis vectors in $\reals^d.$

\subsection{Characterization of federation gains}
\label{subsec: Characterization of Federation gains}
%
%
As mentioned in \prettyref{sec:theta_err}, when $n_j\ll d$, though client $j$ cannot learn any meaningful model based on its local dataset, by joining FL it can learn a model which is a reasonable estimation of $\theta_j^*$ despite heterogeneity.  In this section, we formally characterize, compared with training based on local data only, the gains/loss of a client in joining FL, referred to as {\em federation gain} henceforth. 

Let 
$
\hat{f}_j\equiv \hat{f}_j(\bx_j, y_j)
$
denote any estimator of the true model $f_j^*$ based on the local data $(\bx_j,y_j)$ at client $j$. Let
 \begin{align}
 \label{eq: local risk federation gain}
R^{\mathsf{Loc}}_j= \inf_{\hat{f}_j}
\sup_{f_j^* \in \calH_B}\expects{\lnorm{\hat{f}_j - f_j^*}{\calH}^2}{\bx_j,\xi_j}
\end{align}
denote the minimax risk attainable by the best local estimator $\hat{f}_j$,
where $\calH_B=\{f\in \calH: \lnorm{f}{\calH}\le B\}$\footnote{Here we impose an upper bound $B$ to the RKHS norm of $f_j^*$ to prevent the minimax risk  from  blowing up to the infinity when the local data size $n_j < d$ \cite{mourtada2019exact}.}. 
%
Recall that $f_t$ is the model  trained under FL after $t$ rounds. Consequently, $f_t$ can be viewed as a function of the datasets of all the $M$ clients, i.e.,  $f_t\equiv f_t(\bx, y)$. 
Define the risk of the federated model in estimating $f_j^*$ as 
\begin{align}
\label{eq: federated risk federation gain}
R^{\mathsf{Fed}}_j 
=\inf_{t \ge 0} \sup_{f_j^* \in \calH_B}\expects{\lnorm{f_t - f_j^*}{\calH}^2}{\bx, \xi},
\end{align}
where we take the infimum over time $t$ due to the possible use of the early stopping rule. 
Notably, to average out the randomness induced by the training data, in \eqref{eq: local risk federation gain} and \eqref{eq: federated risk federation gain} the expectations are taken over local training data $\pth{\bx_j, \xi_j}$ and the global training data $\pth{\bx,\xi}$, respectively.

\begin{definition}[Federation gain]
\label{def: federation gain}
The {\em federation gain} of client $j$ in participating FL is defined as 
the ratio of the local minimax risk
and the  federated risk:
$$
\FG_j \triangleq \frac{R^{\mathsf{Loc}}_j}{R^{\mathsf{Fed}}_j }.
$$
\end{definition}
Intuitively, the federation gain is the multiplicative reduction of the error of estimating $f_j^*$ in joining FL compared to the best local estimators.  
Next we give explicit forms of federation gains for the heterogeneity discussed in details in Sections \ref{subsub: isotropic} and \ref{subsec: distinct covariance}.

\begin{theorem}\label{thm:FG}
Consider the same setup as~\prettyref{cor:conv_theta}
and assume that  $\xi_{i} \sim \calN(0,\sigma^2 \identity)$.
Then for $1 \le j \le M,$ 
there exists a constant $c_1$ that only depends on constants $\alpha, \beta$ such that 
\begin{align}
\FG_j
\ge 
  \frac{c_1}{\kappa} \frac{ \min\{ \sigma^2 d/n_j, B^2\} + \max\{1- n_j/d, 0\} B^2 }{\sigma^2  d/N + \Gamma^2 }
\end{align}
\end{theorem}

\prettyref{thm:FG} reveals interesting properties of the federation gain. On the extreme case where $\Gamma=0$ (\ie, there is no model heterogeneity)  
the federation gain
achieves its maximum, which 
is at least on the order of $\min\{N/n_j, N/d\}$. As the model heterogeneity $\Gamma$ increases,
the federation gain decreases. In particular, 
    for data-scarce clients with local data volume $n_j <d$, the federation gain is at least on the order of $(1-n_j/d) B^2/ \Gamma^2$, which exceeds one when $\Gamma \le B \sqrt{1-n_j/d}$.
    For data-rich clients with local data volume $n_j \ge d$, the federation gain is at least on the order of $ \min\{\sigma^2 d/n_j, B^2\}/ \Gamma^2$, which exceeds one when $\Gamma \le  \min\{ \sigma \sqrt{d/n_j}, B\}$.

The following theorem further characterizes the federation gain under the subspace model in the presence of covariate heterogeneity.

\begin{theorem}\label{thm:fg_subspace}
Consider the same setup as \prettyref{cor:subspace}. Further, suppose that $\alpha, \beta$ are fixed positive constants, $f_i^*=f^*$ for all $i \in [M]$, $N \ge C \nu d \log d,$ and $\xi_i \sim \calN(0,\sigma^2 \identity).$
Then for all $1 \le j \le M$, there exists a  constant $c_1>0$ depending on $\alpha, \beta$ such that
\begin{align}
\FG_j 
\ge
 \frac{c_1}{\kappa} \frac{ \min\{ \sigma^2 d/n_j, B^2\} + (1- r_j/d) B^2  }{\sigma^2  d/N}.
\end{align}
\end{theorem}

\prettyref{thm:fg_subspace} implies that when $N \gtrsim \nu d \log d$:
   for data-scarce clients with local data volume $n_j \ll d$, $\FG_j$ is dominated by $N/d$
   which is unchanged with $r_j$;
    for data-rich clients with local data volume $n_j \gg d$, $\FG_j$ is dominated by $\frac{1- r_j/d}{ d/N}$, which is decreasing in $r_j.$
On the contrary, if $N\ll \nu d \log d$,  $f_t$ is not expected to estimate $f^*$ due to the aforementioned coupon collector's problem and hence the federation gain will be small. In conclusion, 
the federation gain will exhibit a sharp jump at a critical sample complexity $N = \Theta( \nu d \log d).$ This is confirmed by our numerical experiment in Section~\ref{app: experiments: data heterogeneity}.  


\section{Experimental Results}
\label{sec: experiments}
In this section, we provide experimental results corroborating our theoretical findings. 

\subsection{Stationary points and estimation errors}
\label{subsec: stationary points v.s. prediction errors}
 We numerically verify that despite the failure of converging to the
 stationary points of the global emprical risk function, both \FedAvg and \FedProx can achieve low estimation errors. 

We adopt the same simulation setup of \cite{pathak2020fedsplit} for fairness in comparison.  
We let $M=25$, $d=100$, and $n_i = 500$. 
For each client $i$, suppose $f_i^* = X_i \theta^* $ for some  $\theta^*\in \reals^d$ and the response vector $y_i \in \reals^{n_i}$ is given by
$y_i = X_i\theta^* + \xi_i, $
where $\xi_i \in \reals^{n_i}$ 
is distributed as $\calN(0, \sigma^2 I)$ with $\sigma = 0.5$. 
The local design matrices $X_i$ are independent random matrices with
i.i.d.\ $\calN(0,1)$ entries. Let $\ell(\theta)=\frac{1}{N} \sum_{i=1}^M \norm{y_i-X_i\theta}^2$ be the global empirical risk function. 
The difference to~\cite{pathak2020fedsplit} is that, instead of plotting the sub-optimality in the excess risk 
$\ell(\theta_t) - 
\min_\theta \ell(\theta)$, we plot the trajectories of $\Norm{\nabla \ell(\theta_t)}$ 
to highlight the unreachability to stationary points of $\ell(\theta_t)$. 
 
%
For both \FedAvg and \FedProx, we choose the step size $\eta=0.1.$
Fig.\,\ref{fig:f1} confirms the observation in \cite{pathak2020fedsplit}
 that \FedAvg with ($s\ge 2$) and \FedProx fail to converge to the stationary point
 of global empirical risk function $\ell(\theta).$
 However, Fig.\,\ref{fig:f2} shows that both \FedAvg with $s=5, 10$ and \FedProx can achieve almost the same low error $\norm{\theta_t-\theta^*}$ as \FedAvg with $s=1$, i.e., 
 the standard centralized gradient descent method.

%
%

\paragraph*{Impact of minibatches} 
Minibatches are often adopted in the real-world implementations of \FedAvg and \FedProx.
Specifically, each client $i$  first partitions its local data into batches of the chosen size $B_i$. Then
for each of the $s$ local steps in \FedAvg \cite{mcmahan2017communication}, the client $i$ updates $\theta_{i,t}$ via running gradient descent $\frac{n_i}{B_i}$ times, 
where a different batch in the data partition is used each time. 
In this way, each of the batches is passed $s$ times in one round. Similarly, in \FedProx \cite{li2018federated}, 
the client $i$ updates $\theta_{i,t}$  
by solving the local proximal optimization~\eqref{eq:Fedprox local obj}
$\frac{n_i}{B_i}$ times, 
where a different batch in the data partition is used each time. 
In our analysis and previous numerical experiments, we assumed full batch $B_i=n_i$. A natural but interesting question is whether \FedAvg and \FedProx
still enjoys the statistical optimality when using minibatches where $B_i<n_i$.



To answer this question, we re-run the experiments with the same setup as above but with three batch sizes $B$: 20, 50, and 100. 
We plot the gradient magnitudes and estimation errors in Fig.\,\ref{app: fig: figure group gradient norm} and~Fig.\,\ref{app: fig: figure group error}. 
For ease of comparison, we redraw Fig.\,\ref{fig:f1} and Fig.\,\ref{fig:f2} in Fig.\,\ref{app:fig:mb1 gradient norm} and Fig.\,\ref{app:fig:mb1 error}. 

As illustrated in Fig.\,\ref{app: fig: figure group gradient norm}, 
for $s=5$ and $s=10$ the impacts of different batch sizes on the gradient magnitude are negligible.
However, strikingly, for \FedAvg $s=1$ with minibatch, its gradient magnitude rises up significantly 
and hence it can no longer reach the stationary point
(This can be rigorously proved by following the arguments in~\cite{pathak2020fedsplit}.).
For \FedProx with minibatch, its curve mostly coincides
with that of \FedAvg $s=1.$
In contrast, as shown in Fig.\,\ref{app: fig: figure group error}, 
the minibatch has almost no effect 
on the estimation error. 
The final estimation errors are almost identical in each of the 
four figures in Fig.\,\ref{app: fig: figure group error}. 
The convergence speed of \FedAvg $s=1$
only decreases a bit with minibatch.

In conclusion, we see  that 
both \FedAvg and \FedProx with minibatch 
can achieve low estimation errors despite the unreachability of the stationary points.

\begin{figure}[h]
  \begin{subfigure}[b]{0.45\textwidth}
    \includegraphics[width=\textwidth]{grad.pdf}
    \caption{Full $\calS_i$ in each local gradient descent update}
    \label{app:fig:mb1 gradient norm}
  \end{subfigure}
    \hfill
  \begin{subfigure}[b]{0.45\textwidth}
    \includegraphics[width=\textwidth]{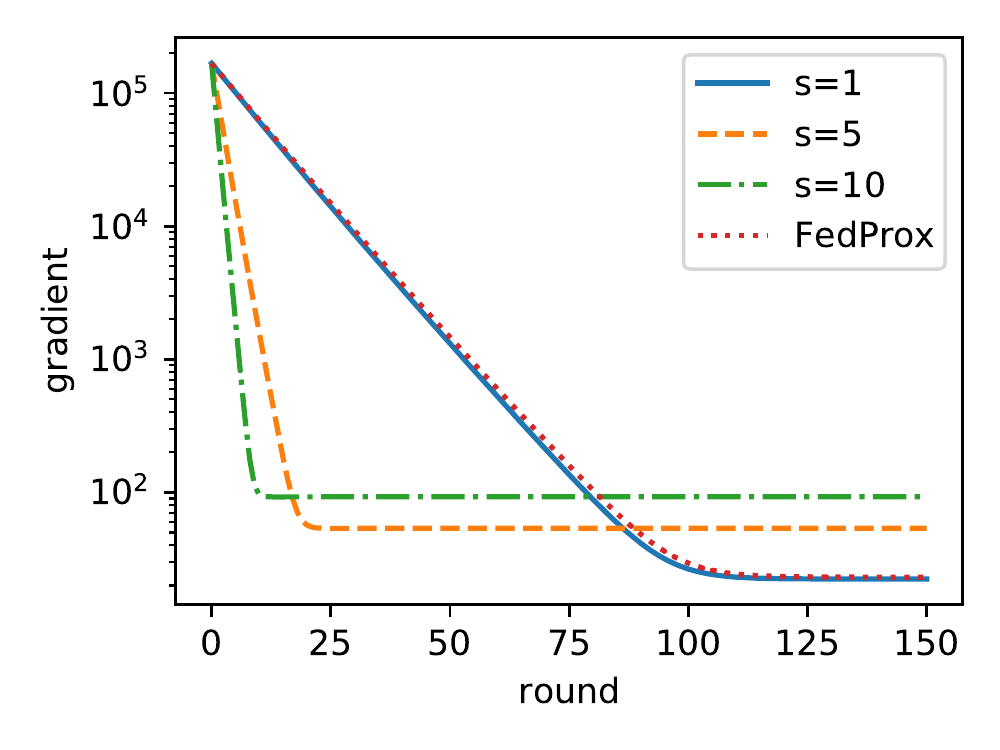}
    \caption{Batch size 20}
    \label{app:fig:mb: gradient norm: 20}
  \end{subfigure}
  \hfill 
    \begin{subfigure}[b]{0.45\textwidth}
    \includegraphics[width=\textwidth]{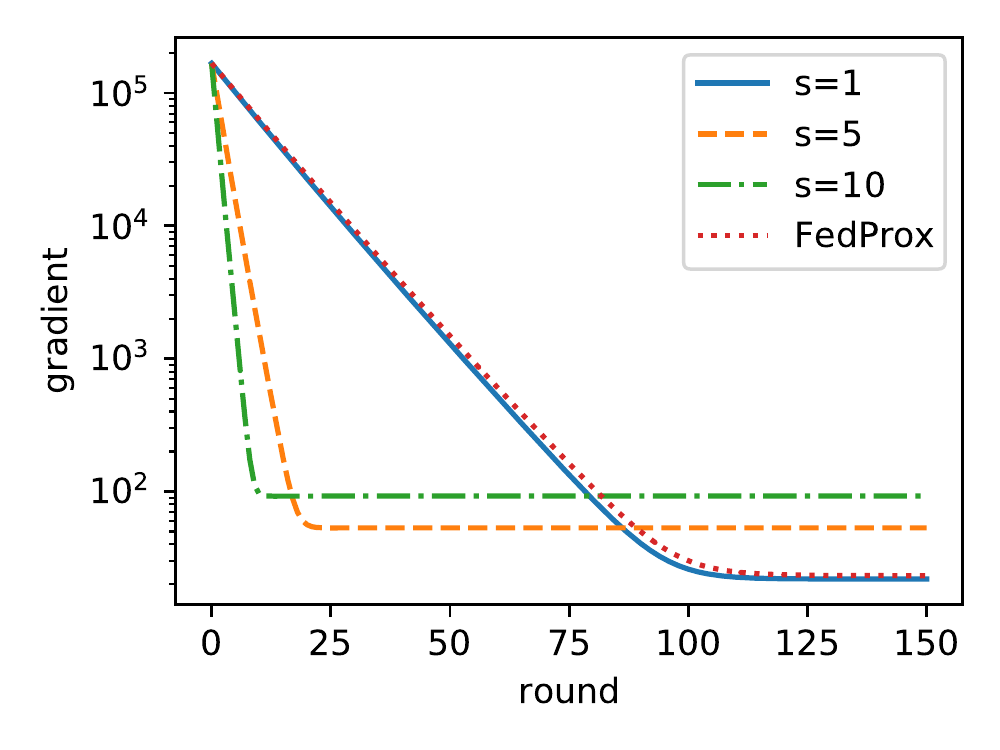} 
    \caption{Batch size 50}
    \label{app:fig:mb: gradient norm: 50}
  \end{subfigure}
  \hfill
    \begin{subfigure}[b]{0.45\textwidth}
    \includegraphics[width=\textwidth]{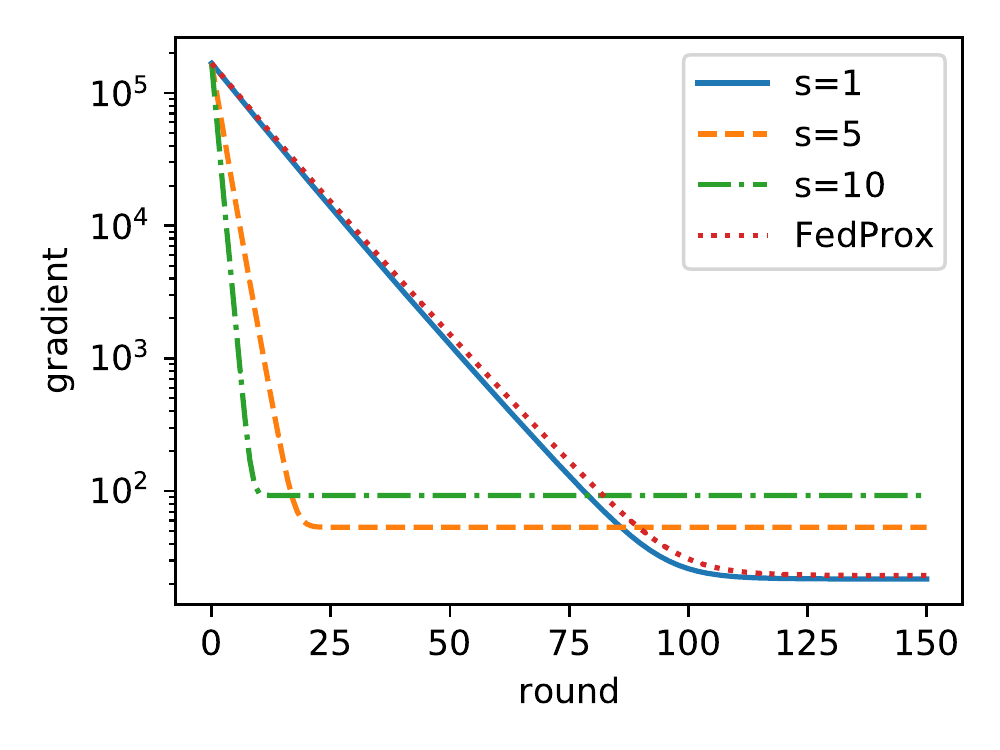}
    \caption{Batch size 100} 
    \label{app:fig:mb: gradient norm: 100}
  \end{subfigure}
\caption{Impacts of mini batch sizes on the reachability of stationary points}
\label{app: fig: figure group gradient norm}
\end{figure}
%
%

%
%
\begin{figure}[h]
  \begin{subfigure}[b]{0.45\textwidth}
    \includegraphics[width=\textwidth]{err.pdf}
    \caption{Full $\calS_i$ in each local gradient descent step}
    \label{app:fig:mb1 error}
  \end{subfigure}
    \hfill
  \begin{subfigure}[b]{0.45\textwidth}
    \includegraphics[width=\textwidth]{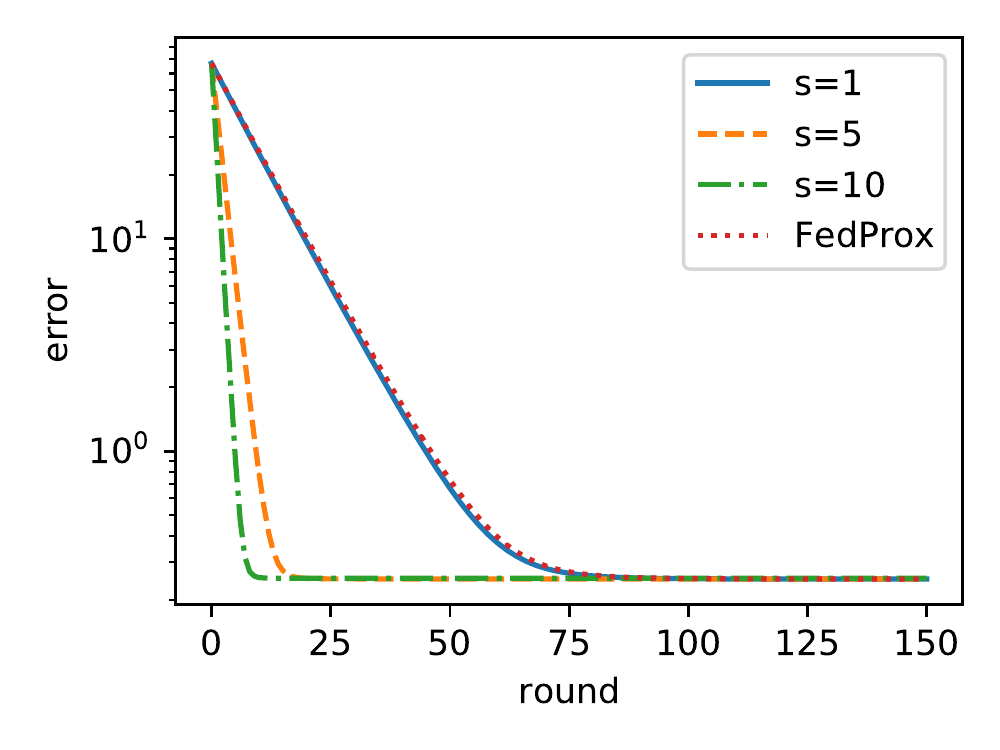}
    \caption{Batch size 20}
    \label{app:fig:mb: err: 20}
  \end{subfigure}
  \hfill 
    \begin{subfigure}[b]{0.45\textwidth}
    \includegraphics[width=\textwidth]{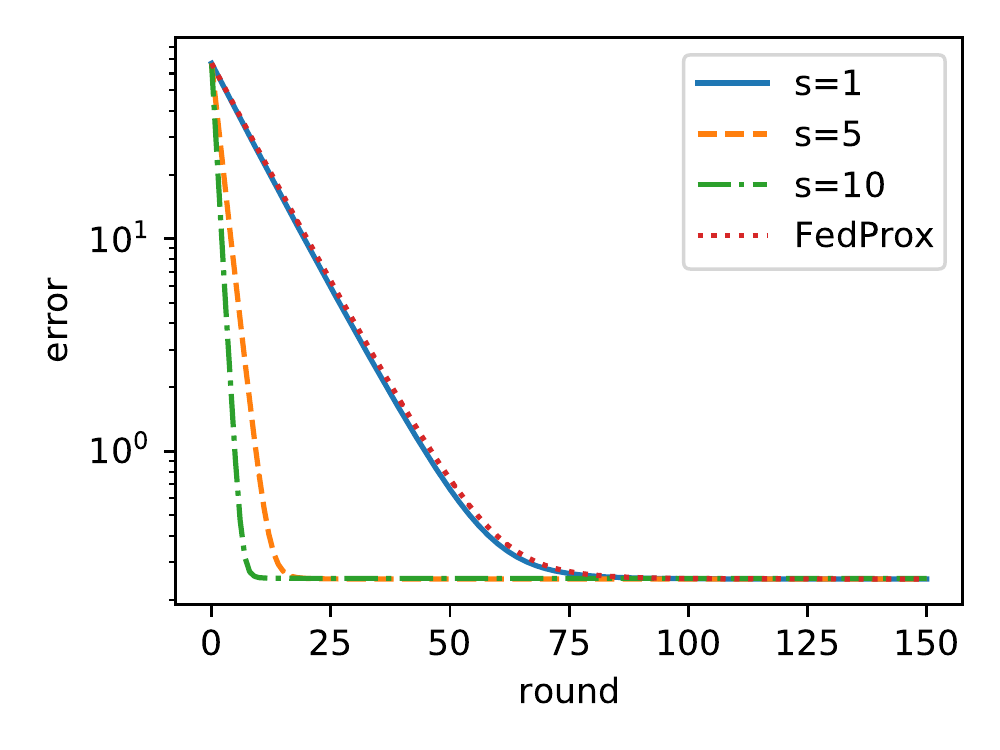} 
    \caption{Batch size 50}
    \label{app:fig:mb: err: 50}
  \end{subfigure}
  \hfill
    \begin{subfigure}[b]{0.45\textwidth}
    \includegraphics[width=\textwidth]{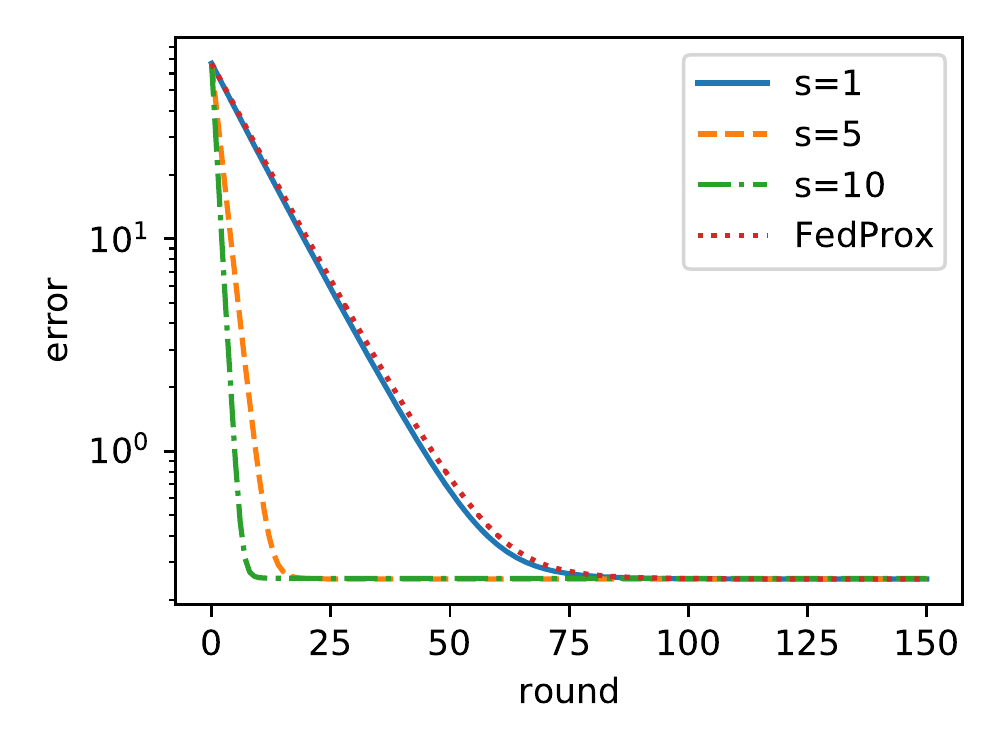}
    \caption{Batch size 100} 
    \label{app:fig:mb: err: 100}
  \end{subfigure}
\caption{Impacts of mini batch sizes on the estimation errors}
\label{app: fig: figure group error}
\end{figure}




%
\subsection{Federation gains versus model heterogeneity}
\label{subsec: experiments gains}
%
%


As mentioned in Section \ref{sec: problem formulation},  data heterogeneity includes both model heterogeneity (a.k.a.\ concept shift) and covariate heterogeneity (a.k.a.\ covariate shift).   
 Complementing our Theorem \ref{thm:FG},  we provide a numerical study on the impact of model heterogeneity on the federation gain in this section, and the corresponding results of covariate heterogeneity in the next section. 
We build on our previous experiment setup 
by allowing  for
unbalanced local data and the heterogeneity in $f^*_i$. 
We choose $M=20$, $d=100$, 
$n_i = 50$ for half of the clients, and $n_i = 500$ for the remaining clients.
We refer to the clients with $n_i=50$ as {\em data scarce} clients, and to the others 
as {\em data rich} clients. 
 We run the experiments with a prescribed set of heterogeneity levels.
All the other specifications are the same as before. 

%
We randomly choose a 
data scarce client and a 
data rich client, 
and plot the federation gains against 
the model heterogeneity $\Gamma=\max_{i,j \in [M]}\|\theta^*_i-\theta^*_j\|_2$ in Fig.\,\ref{fig: figure group 2}.  
Note that in evaluating the federation gains,
we use the minimum-norm least squares as the benchmark local estimator,
 that is $\hat{\theta}_j= (X_j^\top X_j)^+ X_j^\top y_j$, where the symbol $+$
 denotes the Moore-Penrose pseudoinverse. It is known that 
 this estimator can attain the minimax-optimal estimation error rate~\cite{mourtada2019exact}. 
\begin{figure}[h]
  \begin{subfigure}[b]{0.45\textwidth}
    \includegraphics[width=\textwidth]{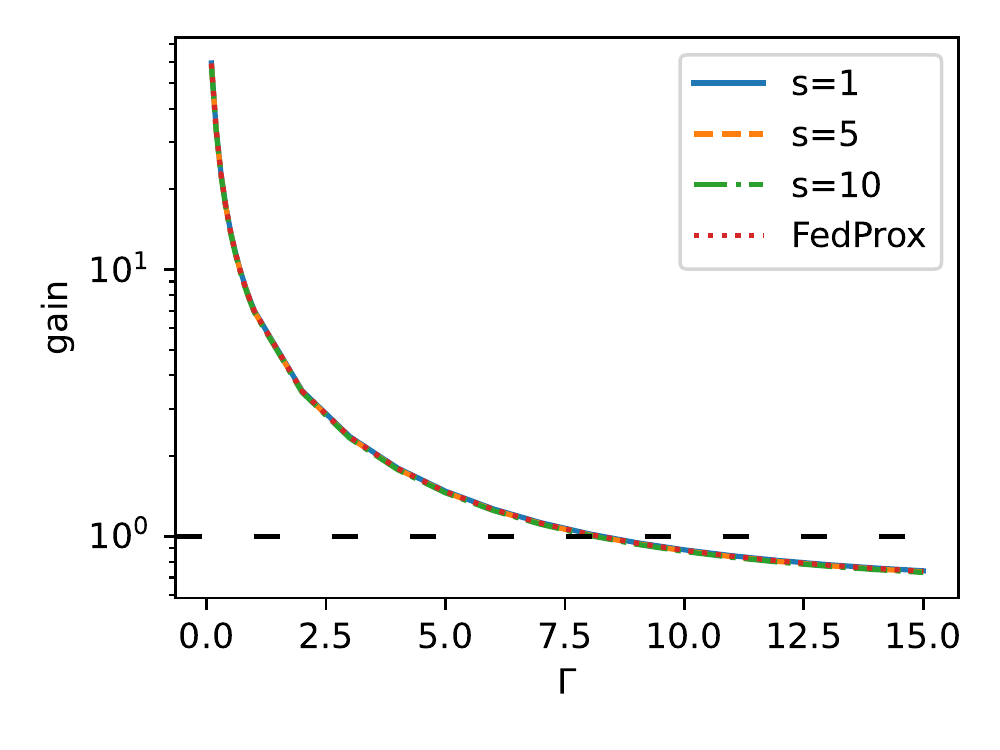} 
    \caption{A data scarce client}
    \label{fig:f3}
  \end{subfigure}
    \hfill
  \begin{subfigure}[b]{0.45\textwidth}
    \includegraphics[width=\textwidth]{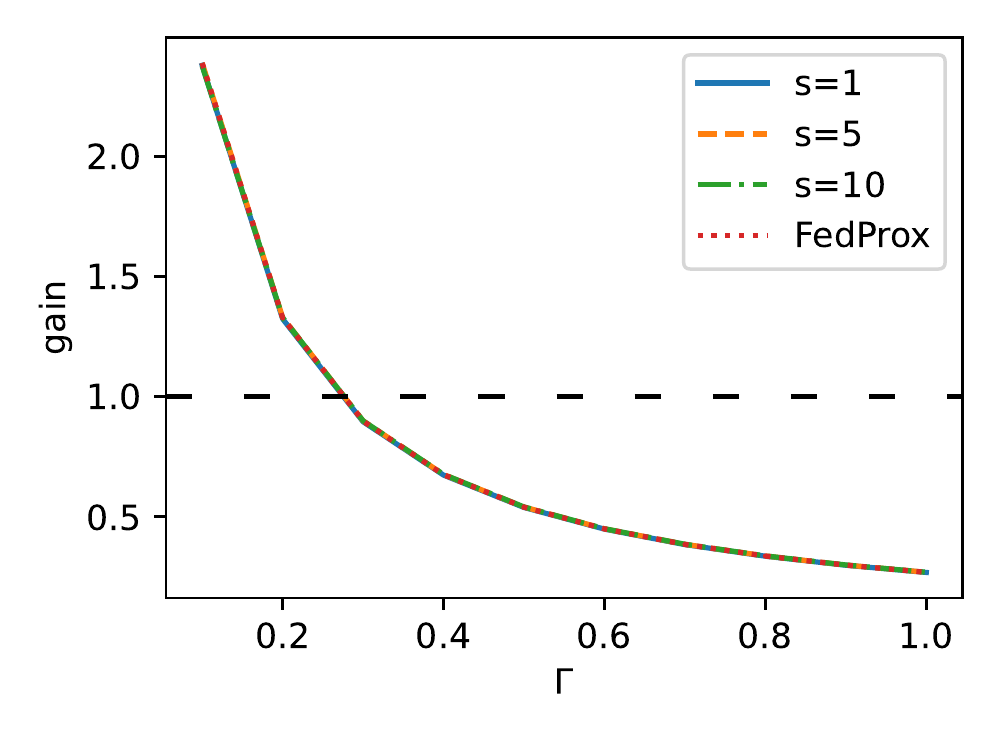}
    \caption{A data rich client}
    \label{fig:f4}
  \end{subfigure}
\caption{Federation gains versus $\Gamma$. 
A data scarce client benefits more from FL participation. }
\label{fig: figure group 2}
\end{figure}
We see that consistent with our theory, 
despite the difference in the training behaviors, the models trained 
under \FedAvg with different choices of aggregation periods $s$ and under \FedProx 
have almost indistinguishable federation gains. 
Moreover, as predicted by our theory, the federation gain drops with increasing model
heterogeneity $\Gamma$, while the federation gain of the data scarce client is much 
higher than that of the data rich client. 
Recall that the federation gain exceeds $1$ 
if and only if the FL model is better than the locally trained model.
We observe that the federation gain of a data scarce client 
drops below 1 at $\Gamma \approx 7.5$, whereas the federation gain of a data rich client drops below 1 at $\Gamma \approx 0.3$. These numbers turn out to be closely match with our theoretically predicted thresholds given after~\prettyref{thm:FG}, 
which are $\Gamma \approx 
 \sqrt{1-n_j/d} \, \|\theta_j^*\|_2 \approx 7$
 and $\Gamma \approx \sigma \sqrt{d/n_j} \approx 0.22$, respectively.

\subsection{Federation gain versus covariate heterogeneity}
\label{app: experiments: data heterogeneity}
%



\begin{figure}[h]
 \begin{subfigure}[b]{0.48\textwidth}
    \includegraphics[width=\textwidth]{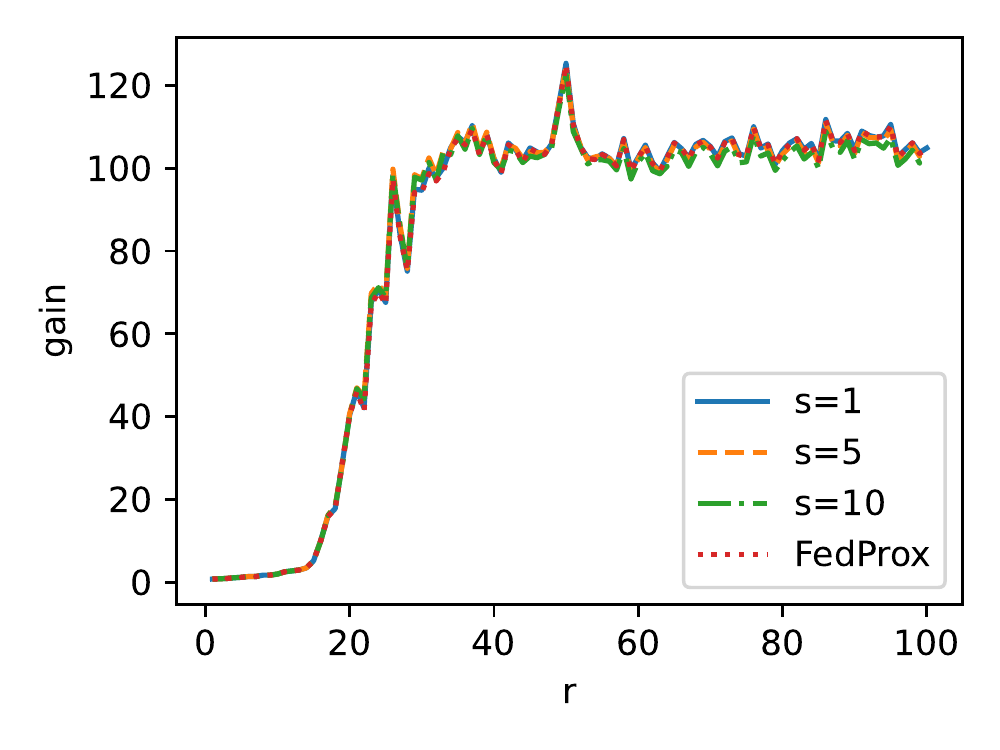}
    \caption{A data scarce client}
    \label{app:fig:subspace: data scarce}
\end{subfigure}  
    \hfill
  \begin{subfigure}[b]{0.48\textwidth}
    \includegraphics[width=\textwidth]{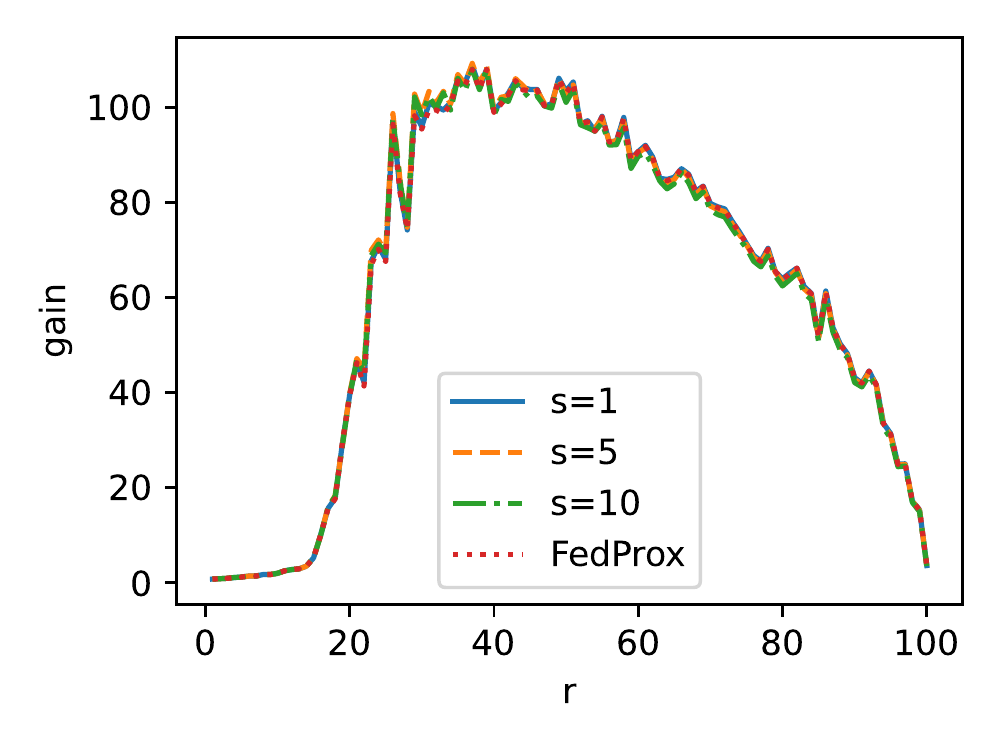}
    \caption{A data rich client}
    \label{app:fig:subspace: data rich}    
  \end{subfigure}
\caption{Federation gains versus subspace dimension $r$. }
\label{app: fig: figure group subspaces}
\end{figure} 

In this section, we study the impact of covariate heterogeneity on the federation gains
by focusing on the subspace model. 
In our experiments, we choose $M=20$, $d=100$, $\sigma=0.5$, $n_i = 50$ for half of the clients, and $n_i = 500$ for the remaining clients. 
We let the 20 clients share a common underlying truth, i.e., $\theta_j^* = \theta^*$ for all $j$, which is randomly drawn from $\calN(0, I)$. 
The responses $y_i$ are given as $y_i = X_i\theta^* + \xi_i$. 
The design matrices $X_i \in \reals^{n_i\times d}$ at the clients lie in different subspaces of dimension $r$; here $r$ ranges from 1 to 100. 
Specifically, $X_i$'s are generated as 
follows: 
we first generate a random index set $E\subseteq [d]$ of cardinality $r$, and generate a matrix ${X}_i$ with each row independently distributed as $\calN(0, (d/r) I_E)$, where $I_E$ is a diagonal matrix with $(I_E)_{ii}=\indc{i\in E}$.
The scaling $\frac{d}{r}$ ensures that each row of $X_i$ has $\ell^2$ norm $\sqrt{d}$ in expectation and hence the signal-to-noise ratio is consistent across different values of $r$.
As $\Norm{X_i}$ increases by a factor of $\sqrt{d/r}$, we rescale the stepsize by choosing $\eta = 0.1/\pth{d/r}$ for the stability of local iterations, according to~\prettyref{cor:subspace}.
Notably, when $r=d$, 
the stepsize becomes $\eta = 0.1$ which is the same as 
previous experiments.
We randomly choose a 
data scarce client and a 
data rich client and record the federation gains. 
We plot the average federation gains over 20 trials against $r$ -- dimension of the subspaces -- in Fig.\,\ref{app: fig: figure group subspaces}. We have the following key observations, matching our theoretical predictions given in~\prettyref{thm:fg_subspace}:
\begin{itemize}
\item 
First, for any fixed $r$, a client's federation gains of the model trained by \FedAvg with $s=1, 5, 10$ and \FedProx are almost identical. This is consistent with our theory, as we show both \FedAvg and \FedProx converge to the minimax-optimal mean-squared error rate $d/N$ in~\prettyref{cor:subspace}. 
\item  Second, up to $r\approx 27$, the curves for the data scarce and the data rich clients are roughly the same. This is because when $r\le 27$, the main ``obstacle'' in learning $\theta^*$ is the lack of sufficient coverage of each of the 100 dimensions by the data collectively kept by the 20 clients. 
\item Third, for both curves there are significant jumps starting when $r\approx 16$ to when $r\approx 23$. 
 If we can pool the data together, 
 due to the coupon-collecting effect, 
 as soon as $M\times r \ge d\log d \approx 460$,
 all the $d$ dimensions can be covered by the design matrices and hence the underlying truth $\theta^*$ can be learned with high accuracy. 
 Since $M=20$, this explains the significant jumps in federations gains in Fig.\,\ref{app: fig: figure group subspaces} when $r$ is around $23$.
\item Finally, the curve trends are different for data scarce and data rich clients. 
For a data scarce client, as shown in Fig.\,\ref{app:fig:subspace: data scarce}, as $r$ increases, the federation gain first increases and then stabilizes around 107. 
In contrast, for a data rich client, as shown in Fig.\,\ref{app:fig:subspace: data rich},  as $r$ increases, the federation gain first increases and then quickly decreases when $r$ approaches 100. 
This distinction is because a data scarce client, on its own, cannot learn $\theta^*$ well as $n_i = 50 \ll 100$ no matter how large $r$ is, 
while a data rich has 500 data tuples and can learn $\theta^*$ on its own quite well when $r$ approaches 100. 
\end{itemize}
%

\subsection{Fitting nonlinear functions}
\label{app: experiments: general regression}
In this section, we go beyond linear models. 
In particular, we focus on fitting $U_5$ -- the degree-5 Chebyshev polynomials of the second kind, which is a special case of the Gegenbauer polynomials and has the explicit expression  
\begin{align*}
    U_5 (x) = 32x^5  - 32 x^3 + 6x. 
\end{align*}
We choose the feature map $\phi(x)=[1,x,\dots,x^5]^\top$ to be the monomial basis up to degree 5 and run \FedAvg and \FedProx on the polynomial coefficients.   
We consider $M=20$ clients and equal size local dataset $n_i \in \{1, 2, \dots, 10\}$. 
Correspondingly, the global dataset size ranges from 20 to 200 as indicated by Fig.\,\ref{app: fig: figure group regression}. 
The response value is given as $y=U_5(x)+\xi$ where $\xi\sim \calN(0,\sigma^2)$ with $\sigma=0.5$. 
We consider heterogeneous local datasets. Specifically, each client $i\in [M]$ probes the function on disjoint intervals $[-1+\frac{2(i-1)}{M},-1+\frac{2i}{M})$.
In the experiments, we generate
covariates $x_{ij}$ using the uniform grid. 
For the fitted function $\hat f$, we evaluate the mean-squared error (MSE) as
\[
\lnorm{\hat{f}-f^*}{2} = \int_{-1}^{1} \left |\hat f(x)-f^*(x)\right |^2 \diff x.
\]
We run \FedAvg and \FedProx with the same stepsize $\eta=0.1$ as before, and evaluate the MSE via Monte Carlo integration. We plot the average MSE over 500 trials.

\begin{figure}[h]
  \begin{subfigure}[b]{0.48\textwidth}
    \includegraphics[width=\textwidth]{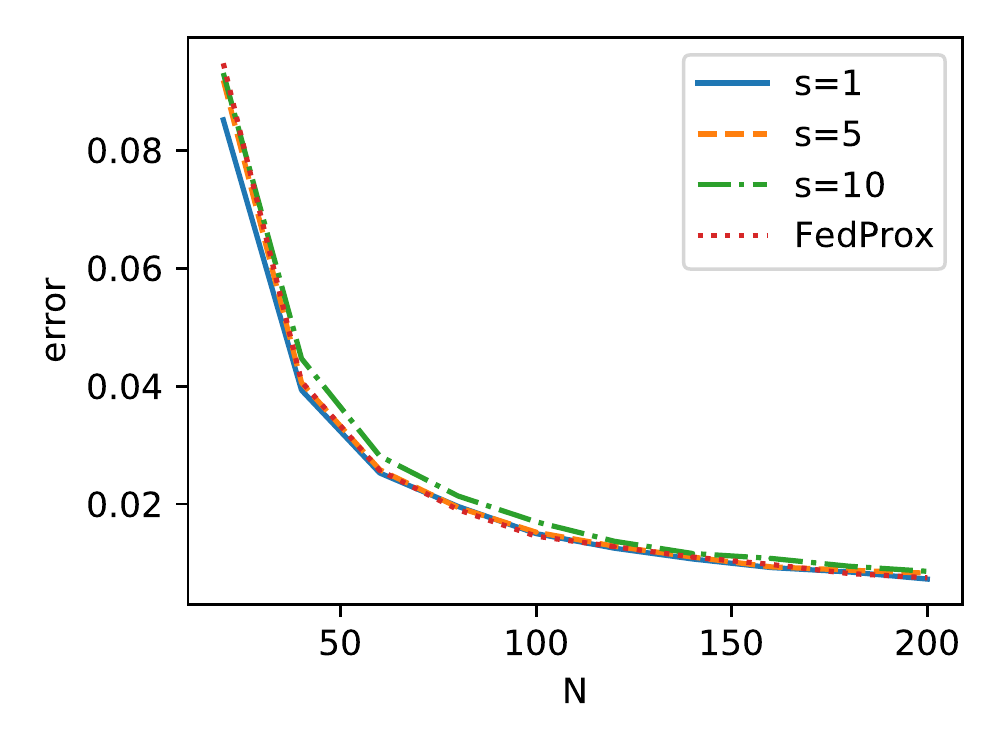}
    \caption{Estimation errors}
    \label{app:fig:regression:error}
  \end{subfigure}
    \hfill
  \begin{subfigure}[b]{0.48\textwidth}
    \includegraphics[width=\textwidth]{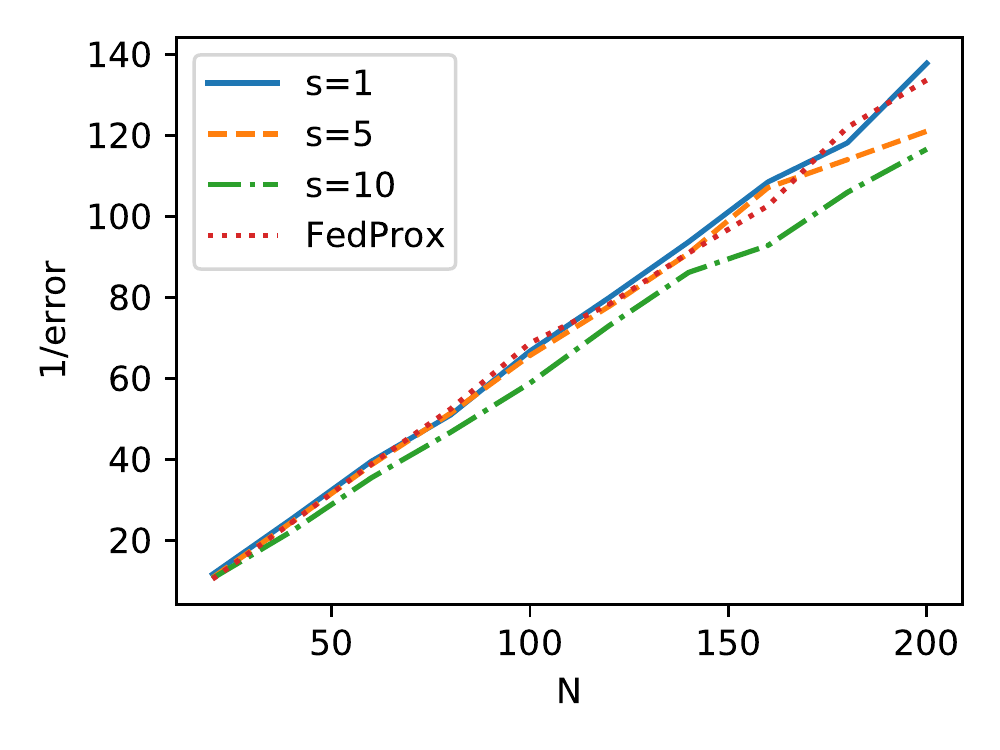} 
    \caption{Reciprocal of estimation errors}
    \label{app:fig:regression:error inverse}    
  \end{subfigure}
\caption{Estimation errors of fitting polynomials}
\label{app: fig: figure group regression}
\end{figure} 
As shown by Fig.\,\ref{app:fig:regression:error}, the four curves of the 
model prediction errors under \FedAvg with different choices of $s$ and \FedProx 
are very similar. Note that for polynomial kernels, the minimax-optimal estimation rate is $O(1/N)$~\cite{raskutti2012minimax}.
In comparison, we plot the reciprocal of 
prediction errors in Fig.\,\ref{app:fig:regression:error inverse}.
Though the differences in 
the reciprocal of the prediction errors  get amplified as the errors approach zero, 
each of the four curves in Fig.\,\ref{app:fig:regression:error inverse} are mostly straight lines.  
Moreover, the four curves have similar slopes with the slope of $s=10$ being slightly smaller than others. 
This is because a larger $s$ leads to a $\kappa$ being slightly greater than one and thus an increased error. These observations confirm that both \FedAvg and \FedProx can achieve nearly-optimal estimation rate in polynomial regression. 

\newcommand{\etalchar}[1]{$^{#1}$}

\appendices

\section{Missing Proofs in Section \ref{sec: error iterate}}
\label{app: error iterate}

\begin{proof}[\bf Proof of \prettyref{prop:theta_recursion}]
For \FedAvg, recall from \eqref{eq:Pi-local} that $\calG_i(f)=\calL_i f + \frac{\eta}{n_i}y_i\cdot k_{\bx_i}$.
Iteratively applying the mapping $\calG_i$ $s$ times, we get that
\[
f_{i,t}=\calG_i^s(f_{t-1}) = \calL_i^s f_{t-1} + \sum_{\tau=0}^{s-1}\calL_i^\tau \frac{\eta}{n_i}y_i\cdot k_{\bx_i}
=\calL_i^s f_{t-1} + y_i \cdot \Psi_i.
\]
Combining the last display with~\prettyref{eq:avg_step} yields \eqref{eq:iteration-f}.

For \FedProx, it follows from~\prettyref{eq:Fedprox local obj} that $f_{i,t}$ solves 
\[
\frac{\eta}{n_i}\sum_{j=1}^{n_i}(f(x_{ij})-y_{ij}) k_{x_{ij}}+(f-f_{t-1})=0.
\]
Then it follows from the definition of the linear operator $\tilde\calL_i^{-1}$ that 
\[
f_{i,t}=\tilde\calL_i^{-1}\pth{f_{t-1}+\frac{\eta}{n_i}\sum_{j=1}^{n_i}y_{ij} k_{x_{ij}}}
=\tilde\calL_i^{-1}f_{t-1} + y_i \cdot \Psi_i. 
\]
Combining the last display with \prettyref{eq:avg_step} yields \eqref{eq:iteration-f} for \FedProx.
\end{proof}

\begin{proof}[\bf Proof of \prettyref{lmm:block_magic}]
We first prove the lemma for \FedAvg.
Note that $f-\calL f=\sum_{i=1}^M w_i (f-\calL_i^s f)$, and that each term admits the following representation 
in terms of telescoping sums:
\begin{align*}
f-\calL_i^s f 
&= \sum_{\tau=0}^{s-1}\calL_i^\tau f -\calL_i^{\tau+1}f
= \sum_{\tau=0}^{s-1}\calL_i^\tau(f-\calL_i f)
=\sum_{\tau=0}^{s-1}\calL_i^\tau\pth{\frac{\eta}{n_i}\sum_{j=1}^{n_i}f(x_{ij})k_{x_{ij}}}=f(\bx_i)\cdot \Psi_i.
\end{align*}
Hence applying the definition of $\Psi$ yields that $f(\bx)\cdot \Psi =\sum_i w_i f(\bx_i) \cdot \Psi_i =f-\calL f$.
Moreover, since 
\begin{equation}
\label{eq:hatL-k}
\calL_i k_{\bx_i}
=
\begin{bmatrix}
k_{x_{i1}}-\frac{\eta}{n_i}\sum_{j=1}^{n_i}k(x_{i1},x_{ij})k_{x_{ij}} \\
\vdots \\ 
k_{x_{in_i}}-\frac{\eta}{n_i}\sum_{j=1}^{n_i}k(x_{in_i},x_{ij})k_{x_{ij}}
\end{bmatrix}
=(I-\eta K_{\bx_i})k_{\bx_i}
,
\end{equation}
it follows from induction that, for every $\tau\in\naturals$, 
\[
\calL_i^\tau k_{\bx_i}=(I-\eta K_{\bx_i})^\tau k_{\bx_i}.
\]
Therefore, using the definition of $P_{ii}$ yields that $\Psi_i=\frac{\eta}{n_i}P_{ii} k_{\bx_i}$
and thus
$
\Psi = \frac{\eta}{N}P k_{\bx}.
$
Since $\Psi(\bx)$ is a $N$ by $N$ matrix that stacks $\Psi(x_{ij})$ in rows,
we obtain that
$
\Psi(\bx)=\eta K_\bx P.
$
Consequently, since $f(x_{ij})-\calL f(x_{ij})
= \Psi(x_{ij}) \cdot f(\bx)
$, 
we have
\[
f(\bx)-\calL f(\bx)
=\Psi(\bx)f(\bx)
=\eta K_\bx P f(\bx).
\]

Analogously, for \FedProx, by the identity 
\[
f-\tilde\calL_i^{-1} f
=\tilde\calL_i^{-1}(\tilde\calL_i f-f)
=\tilde\calL_i^{-1}\frac{\eta}{n_i} \sum_{j=1}^{n_i}f(x_{ij})k_{x_{ij}}
=f(\bx_i)\cdot \Psi_i,
\]
we get that $f(\bx)\cdot \Psi=f-\calL f$.
Similar to \eqref{eq:hatL-k},
\[
\tilde\calL_i k_{\bx_i}
=(I+\eta K_{\bx_i}) k_{\bx_i}.
\]
Therefore, $\tilde\calL_i^{-1} k_{\bx_i}=P_{ii}k_{\bx_i}$ and hence $\Psi_i=\frac{\eta}{n_i}P_{ii} k_{\bx_i}$.
The rest of the proof is identical to that for \FedAvg.
\end{proof}

\vskip \baselineskip 

The following lemma is used in the proof of \prettyref{lmm:eigen-KP}. 
\begin{lemma}
\label{lmm:K-matrix-op-norm}
Suppose $\bx\in\calX^n$ and $\Pi\in\reals^{m\times n }$.
Then, for any $f$ such that $\hnorm{f}\le 1$,  it holds that
\begin{equation}
\label{eq:AKA-psd}
\Pi K_{\bx} \Pi^\top \succeq \frac{1}{n}\Pi f(\bx) f(\bx)^\top \Pi^\top.
\end{equation}
Furthermore, the following is true: 
\begin{equation}
\norm{\Pi K_{\bx} \Pi^\top}
=\max_{\hnorm{f}\le 1}\frac{1}{n}\norm{\Pi f(\bx)}^2.
\end{equation}
\end{lemma}
\begin{proof}
For any $u\in\reals^n$, we have 
\begin{align}
u^\top \Pi K_{\bx} \Pi^\top u
&=\frac{1}{n}\hprod{(u^\top \Pi)\cdot k_{\bx}}{(u^\top \Pi)\cdot k_{\bx}}
= \frac{1}{n}\max_{\hnorm{f}\le 1}\hprod{f}{(u^\top \Pi)\cdot k_{\bx}}^2\nonumber\\
&=\frac{1}{n}\max_{\hnorm{f}\le 1} (u^\top \Pi f(\bx))^2
=\frac{1}{n}\max_{\hnorm{f}\le 1} u^\top (\Pi f(\bx)) (\Pi f(\bx))^\top u,\label{eq:AKA-f}
\end{align}
where the second equality used the fact that $\hnorm{g}=\max_{\hnorm{f}\le 1}\hprod{f}{g}$.
Then \eqref{eq:AKA-psd} follows from \eqref{eq:AKA-f}.
Note that $\Pi K_{\bx} \Pi^\top$ is a symmetric matrix. Then, 
\[
\norm{\Pi K_{\bx} \Pi^\top}
=\max_{\norm{u}\le 1} \frac{1}{n} u^\top \Pi K_{\bx} \Pi^\top u
=\max_{\hnorm{f}\le 1,\norm{u}\le 1}\frac{1}{n}(u^\top \Pi f(\bx))^2
=\max_{\hnorm{f}\le 1} \frac{1}{n}\norm{\Pi f(\bx)}^2, 
\]
where the second equality follows from  \eqref{eq:AKA-f}. 

\qedhere 
\end{proof}

\vskip \baselineskip 

\begin{proof}[\bf Proof of \prettyref{lmm:eigen-KP}]
\underline{\FedAvg:}  
For any $f\in\calH$ we have 
\[
\hprod{f}{\calL_i f}
\overset{(a)}{=}\hnorm{f}^2-\frac{\eta}{n_i}\norm{f(\bx_i)}^2
\le \hnorm{f}^2, 
\]
where equality (a) follows from the definition of $\calL_i$, proving $\opnorm{\calL_i}\le 1$. 
We show $\calL_i$ is positive by lower bounding $\hprod{f}{\calL_i f}$ as follow: 
\[
\hprod{f}{\calL_i f}
=\hnorm{f}^2-\frac{\eta}{n_i}\norm{f(\bx_i)}^2
\overset{(a)}{\ge} \hnorm{f}^2(1-\eta\norm{K_{\bx_i}})
\overset{(b)}{\ge} \hnorm{f}^2(1-\gamma)\overset{(c)}{>}0, 
\]
where inequality (a) follows from \prettyref{lmm:K-matrix-op-norm}, (b) holds by definition of $\gamma$, and (c) is true because that   $\gamma<1$. 
Hence, we obtain that $\calL_i$ is positive with $\opnorm{\calL_i}\le 1$, which immediately implies that both $\calL_i^s$ and $\calL=\sum_i w_i \calL_i^s$ are positive and their operator norm is upper bounded by $1$. 

\underline{\FedProx:} 
For any $f\in \calH$, let $g\triangleq \tilde{\calL}_if$. Next we show that $0\le \Hprod{\tilde\calL_i^{-1} f}{f} \le \hnorm{f}^2$. By definition, it holds that $\hprod{f}{\tilde \calL_i^{-1} f} =\hprod{\tilde\calL_i g}{g}$. We have 
\begin{align*}
\hprod{\tilde\calL_i g}{g} = \hprod{g+\frac{\eta}{n_i}\sum_{j=1}^{n_i}g(\bx_{ij}) k_{\bx_{ij}}}{g} =\hnorm{g}^2 + \frac{\eta}{n_i}\norm{g(\bx_i)}^2 \ge 0.     
\end{align*}
In addition, we have 
\begin{align*}
\hnorm{f}^2 & = \hprod{\tilde\calL_i g}{\tilde\calL_ig}   = \hnorm{g}^2 + 2\frac{\eta}{n_i}\norm{g(\bx_i)}^2 + \pth{\frac{\eta}{n_i}}^2\sum_{j=1}^{n_i}\sum_{j^{\prime}=1}^{n_i} g(x_{ij})g(x_{ij^{\prime}})k(x_{ij}, x_{ij^{\prime}}) \\
& \overset{(a)}{\ge} \hnorm{g}^2 + 2\frac{\eta}{n_i}\norm{g(\bx_i)}^2 \ge \hnorm{g}^2 + \frac{\eta}{n_i}\norm{g(\bx_i)}^2 = \hprod{\tilde\calL_i g}{g}, 
\end{align*}
where inequality (a) is true because the kernel function $k$ is positive semi-definite. 
Hence, we obtain that $\tilde\calL_i^{-1}$ is positive with $\opnorm{\tilde\calL_i^{-1}}\le 1$ regardless of $\gamma$, and so is $\calL=\sum_i w_i\tilde\calL_i^{-1}$.

Note that $P$ is positive definite when $\gamma<1$ for \FedAvg and is positive definite regardless of $\gamma$ for \FedProx. 
Positive definiteness ensure that the matrix $K_{\bx}P$ is similar\,\footnote{Recall that two matrices $A, B\in \reals^{n\times n}$ are similar if there exists an invertible matrix $Q$ such that $B=Q^{-1}AQ$.} to $P^{1/2} K_{\bx} P^{1/2}$ which has non-negative eigenvalues only. So it suffices to prove $\Norm{\eta P^{1/2} K_{\bx} P^{1/2}}\le 1$.
By \prettyref{lmm:K-matrix-op-norm},
\[
\norm{\eta P^{1/2} K_{\bx} P^{1/2}}
= \max_{\hnorm{f}\le 1}\frac{\eta}{N}f(\bx)^\top Pf(\bx)
= \max_{\hnorm{f}\le 1}\frac{\eta}{N}\hprod{f}{(f(\bx)^\top P) \cdot (k_{\bx})}.
\]
Applying \prettyref{lmm:block_magic} yields that
\[
\frac{\eta}{N}\hprod{f}{f(\bx)\cdot (Pk_{\bx})}
=\hprod{f}{f(\bx)\cdot \Psi}
=\hprod{f}{f-\calL f}
=\hnorm{f}^2-\hprod{f}{\calL f}
\le \hnorm{f}^2,
\]
where the last inequality holds because that  
$\calL$ is positive. 
\end{proof}

\vskip \baselineskip

\begin{proof}[\bf Proof of \prettyref{lmm:condition-P}]
By the block structure, it holds that  $\|P^{-1}\|_2=\max_{i\in[M]}\|P_{ii}^{-1}\|_2=\frac{1}{\lambda_{\min}(P_{ii})}$ and $\|P\|_2=\max_{i\in[M]}\|P_{ii}\|_2$.
For \FedAvg, we have $\|P_{ii}\|_2 \le s$, and
\[
\lambda_{\min}(P_{ii})\ge 
\sum_{\tau=0}^{s-1}(1-\gamma)^\tau
=\frac{1-(1-\gamma)^s}{\gamma}.
\]
For \FedProx, we have $\|P_{ii}\|_2\le 1$ and $\lambda_{\min}(P_{ii})\ge (1+\gamma)^{-1}$. Finally, since the eigenvalues of $K_\bx P$ coincide with those of $P^{1/2}K_\bx P^{1/2}$, \prettyref{eq:eigen_comp} follows from Ostrowski's inequality (see e.g.~\cite[Theorem 4.5.9]{horn2012matrix}) and the fact that 
$s/\kappa \le \lambda_{\min}(P_{ii})\ \le \|P_{ii}\|_2 \le s $ for \FedAvg and 
$1/\kappa \le \lambda_{\min}(P_{ii})\ \le \|P_{ii}\|_2 \le 1$ for \FedProx.
\end{proof}


\section{Proofs in Section \ref{sec:pred_error_emp} }
\label{app: convergence}
In this section, we present the missing proofs of results in Section  \ref{sec:pred_error_emp}. 
We focus on proving the results for \FedAvg. The proof 
for \FedProx follows verbatim using the facts that $\norm{P} \le 1$ and that $\norm{P^{-1}} \le \kappa.$

One key idea is to apply the eigenvalue decomposition of $K_{\bx}P$
and to project $f(\bx)$ to the eigenspace of $K_{\bx} P$ for $f\in\calH$. 
We first describe the eigen-decomposition and present bounds on relevant matrix norms.
Recall that 
$P$ is positive-definite, and thus
$K_{\bx}P$ is similar to $P^{1/2} K_{\bx} P^{1/2}$, whose eigenvalue decomposition is denoted as 
\begin{align}
P^{1/2} K_{\bx} P^{1/2} = U \Lambda U^\top, \label{eq:EVD_PKP}
\end{align}
where $U$ is unitary, i.e., $U^{\top} = U^{-1}$, 
and $\Lambda=\diag{\Lambda_{1},\dots,\Lambda_{N}}$ 
is a $N \times N$ diagonal matrix with non-negative entries. 
Let $V \triangleq P^{-1/2} U$ and $L\triangleq I-\eta K_{\bx} P$.
Then,
\begin{equation}
\label{eq:K-eigen}
K_{\bx} = V \Lambda V^{\top},
\qquad K_{\bx} P= V\Lambda V^{-1},
\qquad L^t = V(I-\eta\Lambda)^t V^{-1}.
\end{equation}
By the definition of $V$,
\begin{equation}
\label{eq:norm-V}
\norm{V}^2=\norm{P^{-1}}\le \kappa/s, 
\qquad \norm{V^{-1}}^2=\norm{P}\le s,
\end{equation}
where the upper bounds of $\norm{P^{-1}}$ and $\norm{P}$ are derived in the proof of \prettyref{lmm:condition-P}.
Since $\gamma<1$, \prettyref{lmm:eigen-KP} shows that $\Norm{\eta \Lambda}\le 1$ and thus 
$\Norm{I-(I-\eta \Lambda)^t}\le 1$. Therefore, using the eigenvalue decomposition in \eqref{eq:K-eigen} and the upper bounds in \eqref{eq:norm-V}, we have
\begin{equation}
\label{eq:norm-I-Lt}
\norm{I-L^t}
=\norm{V(I-(I-\eta \Lambda)^t)V^{-1}}
\le \sqrt{\kappa}.
\end{equation}

\begin{lemma}
\label{lmm:Fnorm-I-Lt}
\[
\Fnorm{I-L^t}
\le \sqrt{\kappa}\Fnorm{I-(I-\eta \Lambda)^t}
\le \sqrt{\kappa N\eta t s} \calR \left( \frac{1}{\sqrt{\eta ts}} \right),
\]
where $\calR$ is the empirical Rademacher complexity defined in \eqref{eq:R_K_def}.
\end{lemma}
\begin{proof}
Applying the the eigenvalue decomposition \eqref{eq:K-eigen} and the inequality $\Fnorm{AB}\le \Norm{A}\Fnorm{B}$ (see, e.g., \cite[5.6.P20]{horn2012matrix}), we obtain
\[
\Fnorm{I-L^t}
= \Fnorm{V(I-(I-\eta \Lambda)^t)V^{-1}}
\le \norm{V}\norm{V^{-1}}\Fnorm{I-(I-\eta \Lambda)^t}
\le \sqrt{\kappa}\Fnorm{I-(I-\eta \Lambda)^t},
\]
where the last inequality used \eqref{eq:norm-V}.

Next we prove the second inequality.
Since $0\le \eta\Lambda_i\le 1$, it holds that
\begin{align*}
    \Fnorm{I-(I-\eta \Lambda)^t}^2
    &=\sum_{i=1}^N \left(  1- \left( 1- \eta \Lambda_{i} \right)^t \right)^2 \\
    & \le \sum_{i=1}^N \min\left\{  1, \eta^2 t^2 \Lambda_{i}^2 \right\} \\
    & \le \sum_{i=1}^N \min\left\{  1, \eta t \Lambda_{i} \right\}\\
    & \le  \sum_{i=1}^N \min\left\{  1, \eta t \lambda_i s \right\},
\end{align*} 
where the last inequality holds because $\Lambda_{i} \le \lambda_i \|P\| \le \lambda_i s $ in view of~\prettyref{eq:eigen_comp}. 
The conclusion follows from the definition of $\calR$ in~\prettyref{eq:R_K_def}.
\end{proof}


\subsection{Proof of~\prettyref{prop:pred_error_general}} \label{sec:conv_f}
%
We show the convergence of the prediction error~\prettyref{eq:pred_L2}.
Plugging $y=f(\bx)+\Delta_f+\xi$ into~\prettyref{eq:pred_error_recursion_1}, we get that
$$
f_{t}(\bx)= \left[I - \eta K_\bx P\right]f_{t-1}(\bx) + \eta K_\bx P \left( f(\bx)+\Delta_f+\xi \right).
$$
Subtracting $f(\bx)$ from both hand sides yields that
$$
f_{t}(\bx) - f(\bx) = \left[I - \eta K_\bx P\right] \left(
f_{t-1}(\bx) - f(\bx)
\right) + \eta K_\bx P \left(\Delta_f+\xi \right).
$$
%
Unrolling the above recursion and recalling $L\triangleq I-\eta K_{\bx} P$, we deduce that
\begin{align}
 f_t(\bx) - f(\bx) &= L^t  (f_0(\bx) -f(\bx) ) + \sum_{\tau=0}^{t-1} L^\tau \eta K_\bx P \left(\xi+ \Delta_{f}\right) 
 \nonumber \\
 & = L^t  (f_0(\bx)-f(\bx) ) + \left( I - L^t \right) \left(\xi+ \Delta_{f}\right),
 \label{eq:pred_error_expansion delta}
\end{align}
where the last equality follows from the identity that 
$\sum_{\tau=0}^{t-1} (I- A)^\tau A = I- (I-A)^t.$
It follows that 
\begin{align}
 \norm{f_t(\bx) - f(\bx)}^2  & \le 3\norm{L^t   (f_0(\bx) -f(\bx) )}^2 + 3 \norm{(I - L^t )\xi}^2
+3 \norm{(I - L^t ) \Delta_{f}}^2  \nonumber \\
&
\le 3\norm{L^t   (f_0(\bx) -f(\bx) )}^2 + 3 \norm{(I - L^t )\xi}^2
+3 \kappa \norm{\Delta_{f}}^2,
\label{eq: thm1: obj}
\end{align}
where the last inequality holds due to~\prettyref{eq:norm-I-Lt}.
To finish the proof of~\prettyref{prop:pred_error_general},
it suffices to apply
the following two lemmas, 
which bound the 
first (bias) and  the second  (variance)
terms in \eqref{eq: thm1: obj}, respectively. 


\begin{lemma}[Bias]\label{lmm:bias}
For all iterations $t=1,2, \ldots,$ it holds that 
 $$
 \frac{1}{N} \norm{L^t  (f_0(\bx) -f(\bx) )}^2
 \le \kappa \delta_1(t) \lnorm{f_0-f}{\calH}^2,
 $$
where 
 $$
\delta_1(t) \triangleq \frac{1}{s} \max_{1 \le i \le N} \left( 1 - \eta \Lambda_{i} \right)^{2t} \Lambda_{i} \le \frac{1}{2e \eta ts }.
 $$
\end{lemma}

\begin{proof}
For any $f\in\calH$, it follows from \prettyref{lmm:K-matrix-op-norm} that $\frac{1}{N}f(\bx)f(\bx)^\top \preceq \lnorm{f}{\calH}^2 K_{\bx}$.
Then, 
\[
\frac{1}{N}\norm{L^t  f(\bx)}^2
= \frac{1}{N}\norm{L^t f(\bx)f(\bx)^\top (L^t)^\top}
\le \hnorm{f}^2\norm{L^t K_\bx (L^t)^\top}.
\]
Applying \eqref{eq:K-eigen} yields that
\begin{align*}
\norm{L^t K_\bx (L^t)^\top}
& =\norm{V[(I-\eta \Lambda)^{2t}\Lambda]V^\top}
\le \norm{(1-\eta\Lambda)^{2t}\Lambda}\norm{V}^2
\le \norm{(1-\eta\Lambda)^{2t}\Lambda} \kappa/s,
\end{align*}
where the last inequality used \eqref{eq:norm-V}.
The conclusion follows by noting that $f-f_0\in\calH$, $\Norm{\eta \Lambda}\le 1$ by \prettyref{lmm:eigen-KP}, and $(1-x)^t x \le \frac{1}{et}$ for all $x  \le 1$.
\end{proof}

\begin{lemma}[Variance]\label{lmm:variance}
For all iterations $t=1, 2, \ldots,$ it holds that
\begin{equation}
\label{eq:pred_error_sample_expect_1}
\frac{1}{N} \expect{\norm{( I - L^t)\xi}^2}
\le \kappa \sigma^2  \delta_2(t), 
\end{equation}
where
$$
\delta_2(t) \triangleq  \frac{1}{N} \sum_{i=1}^N \left( 1 - \left( 1- \eta \Lambda_{i} \right)^{t} \right)^2 
\le \frac{1}{N} \sum_{i=1}^N \min \left\{1, \eta t \Lambda_{i} \right\} \le \eta t s \calR^2 \left( \frac{1}{\sqrt{\eta ts}}\right).
$$
\end{lemma}
\begin{proof}
Note that
\begin{align}
\label{eq: lm: variance: noise}
\norm{(I - L^t) \xi}^2
= \xi^\top Q \xi 
= \Tr\left( \xi \xi^\top Q \right),
\end{align}
where $Q=(I - L^t)^\top (I - L^t)\succeq 0$.
By the assumption $\expect{\xi \xi^\top} \preceq \sigma^2 I$ and the fact 
\begin{equation}
\label{eq:trace_fact}
\Tr(YZ) \ge 0, \quad \text{if $Y\succeq 0$ and $Z\succeq 0$},
\end{equation}
we have
\begin{equation}
\label{eq:expect-noise-norm-ub}    
\Expect\Norm{(I - L^t) \xi}^2 =\Tr\left(\Expect [\xi \xi^\top] Q \right)
\le \sigma^2 \Tr(Q)
= \sigma^2\Fnorm{I-L^t}^2.
\end{equation}
Then \eqref{eq:pred_error_sample_expect_1} follows from \prettyref{lmm:Fnorm-I-Lt}.
\end{proof}



\subsection{Proofs of Theorems~\ref{thm:pred_error_emp} --~\ref{thm: exponential} }
\label{sec:proof_corollaries}

We first deduce the convergence result with early stopping from~\prettyref{prop:pred_error_general}.
\begin{proof}[Proof of~\prettyref{thm:pred_error_emp}]
Plugging the upper bounds of $\delta_1(t) $ and $\delta_2(t)$ in Lemma~\ref{lmm:bias} and Lemma~\ref{lmm:variance} into~\prettyref{eq:error_bound_general} in~\prettyref{prop:pred_error_general},
we get that 
\begin{align*}
\frac{1}{N} \expects{\norm{f_t(\bx) - f(\bx) }^2}{\xi}
&\le \frac{3\kappa}{2e\eta ts} \lnorm{f_0-f}{\calH}^2 + 3\kappa \sigma^2 \eta ts \calR^2\left( \frac{1}{\sqrt{\eta ts}} \right) 
 + \frac{3\kappa}{N}\norm{\Delta_{f}}^2\\
&\le \frac{3 \kappa }{2e\eta ts} \left( \lnorm{f_0-f}{\calH}^2 +1 \right)  + \frac{3\kappa}{N}\norm{\Delta_{f}}^2, \quad \forall 1 \le t \le T,    
\end{align*}
where the last inequality holds because by the definition of early stopping time $T$ given in~\prettyref{eq:def_early_stopping},
we have
$\eta ts \calR(1/\sqrt{\eta ts}) \le 1/(\sqrt{2e} \sigma)$ for all $t \le T$ given that $\calR(\epsilon)/\epsilon^2$ is non-increasing in $\epsilon$.
\end{proof}



Then we deduce the convergence results that hold with high probability. 


\begin{proof}[Proof of~\prettyref{thm: Light-tailed noises}]
For any $t\ge 0$, let
$$
q_t \triangleq \prob{\frac{1}{N} \norm{\left[ I - L^t \right] \xi}^2
\ge \frac{1}{N} \expect{ \norm{\left[ I - L^t \right] \xi}^2}
+ \frac{\delta}{N} },
$$
where $\delta=N \kappa/(e\eta ts)$. 
By definition of $\varepsilon_t$ in~\prettyref{eq:def_epsilon_t} and~\prettyref{thm:pred_error_emp}, we have
$\varepsilon_t \le q_t.$
Recall from~\eqref{eq: lm: variance: noise} that 
$\norm{(I - L^t) \xi}^2
= \Tr\left( \xi \xi^\top Q \right) $. 
It remains to show the concentration of the quadratic expression $ \Tr\left( \xi \xi^\top Q \right)$.
\paragraph{Sub-Gaussian noise. }
Using Hanson-Wright's inequality~\cite{rudelson2013hanson}, we get 
\begin{equation}
\label{eq:Hanson-Wright}    
\prob{ \Tr\left( \xi \xi^\top Q \right) - 
\expect{\Tr\left( \xi \xi^\top Q \right)} \ge \delta}
\le \exp \left( -c_1 \min \left\{ \frac{\delta}{\sigma^2\norm{Q}}, \frac{\delta^2}{\sigma^4 \fnorm{Q}^2} \right\} \right),
\end{equation}
where $c_1>0$ is a universal constant. 
Note that 
\begin{align}
\label{eq: bound Q}
\norm{Q} 
&\le  \norm{V^{-1}}^2 \norm{V}^2 
= \norm{P}\norm{P^{-1}} \le \kappa,\\
\label{eq:bound Q-F}
\fnorm{Q}^2
&=\Tr(QQ^\top) \le \norm{Q} \Tr(Q),
\end{align}
where the last inequality follows from~\prettyref{eq:trace_fact}.
Applying~\prettyref{lmm:Fnorm-I-Lt} yields that 
\begin{align*}
&\frac{\delta^2}{\sigma^4 \fnorm{Q}^2}
\ge \frac{\delta}{\sigma^2\norm{Q}}\frac{\delta}{\sigma^2 \Tr(Q)}
\ge \frac{\delta}{\sigma^2\norm{Q}}  \frac{\delta}{\sigma^2 \kappa N\eta ts \calR^2\pth{\frac{1}{\sqrt{\eta t s}}} },\\
& \frac{\delta}{\sigma^2\norm{Q}}
\ge \frac{\delta}{\sigma^2 \kappa}.
\end{align*}
Thus we obtain that
\begin{equation}
\label{eq:HW-prob}
\min \left\{ \frac{\delta}{\sigma^2\norm{Q}}, \frac{\delta^2}{\sigma^4 \fnorm{Q}^2} \right\}
\ge \frac{\delta}{\sigma^2 \kappa} \min \left\{ 1, 
\frac{\delta}{\sigma^2 \kappa N\eta ts \calR^2\pth{\frac{1}{\sqrt{\eta t s}}} }\right\}.
\end{equation}
 Combining~\eqref{eq:expect-noise-norm-ub}, \eqref{eq:Hanson-Wright}, and \eqref{eq:HW-prob} yields that 
 \begin{align}
q \le \exp \left( - 
\frac{c_1 \delta}{\sigma^2 \kappa} \min \left\{ 1, 
\frac{\delta}{\sigma^2 \kappa N\eta ts \calR^2\pth{1/\sqrt{\eta t s}} }\right\}
\right).
\label{eq:pred_error_sample_hp_1}
\end{align}
Recalling $\delta=N \kappa/(e\eta ts)$, we deduce that for $t \le T,$
\begin{align*}
q & \le \exp \left( - 
\frac{c_1 N }{\sigma^2 e \eta ts} \min \left\{ 1, 
\frac{1}{\sigma^2 e (\eta ts)^2 \calR^2\pth{\frac{1}{\sqrt{\eta t s}}} }\right\} 
\right) \\
& \le \exp \left( - \frac{c_1 N }{\sigma^2 e \eta ts}\right),
\end{align*}
where the last inequality holds due to $\eta t s \calR(1/\sqrt{\eta t s}) \le 1/(\sqrt{2e}\sigma)$ for $t \le T$.

\paragraph{Heavy-tailed noise. }
We first prove a concentration inequality analogous to the Hanson-Wright inequality. Note that 
$\Tr\left( \xi \xi^\top Q \right)=\xi^\top Q \xi$. 
We decompose the deviation into two parts and bound their tail probabilities separately:
\begin{align}
\Prob\qth{\abs{\xi^\top Q \xi - \Expect[\xi^\top Q \xi]}>\delta}
&\le \prob{\abs{\sum_{i} Q_{ii}(\xi_i^2-\Expect \xi_i^2)}>\frac{\delta}{2}}
+ \prob{\abs{\sum_{i\ne j}Q_{ij}\xi_i\xi_j}>\frac{\delta}{2}}\nonumber\\
&\le \frac{\Expect\abs{\sum_{i} Q_{ii}(\xi_i^2-\Expect \xi_i^2)}^{p/2}}{(\delta/2)^{p/2}}
+ \frac{\Expect|\sum_{i\ne j}Q_{ij} \xi_i\xi_j|^p}{(\delta/2)^{p}}\label{eq:tail-moment}.
\end{align}
The first term involves a sum of independent random variables. 
Since $\Expect(\xi_i^2-\Expect \xi_i^2)^2\le \Expect|\xi_i|^4\triangleq M_4$ and $\Expect|\xi_i^2-\Expect \xi_i^2|^{p/2}\le 2^{p/2}\Expect|\xi_i|^p$,
by the Rosenthal-type inequality \cite[Theorem 5.2]{pinelis1994},
\begin{equation}
\label{eq:moment-diag}
\Expect\abs{\sum_{i} Q_{ii}(\xi_i^2-\Expect \xi_i^2)}^{p/2}
\le C_p\pth{M_p\sum_{i}|Q_{ii}|^{p/2}+\pth{M_4\sum_i |Q_{ii}|^2}^{p/4}}
\le 2C_p\Fnorm{Q}^{p/2}M_p,
\end{equation}
where $C_p$ only depends on $p$, and we used $M_4^{1/4}\le M_p^{1/p}$ and $\|x\|_p\le \|x\|_q$ for $p\ge q$.
For the second term, the decoupling inequality gives \cite{de1995}:
\[
\Expect\abs{\sum_{i\ne j}Q_{ij} \xi_i\xi_j}^p
\le 4^p\Expect\abs{\xi^\top Q \xi'}^p,
\]
where $\xi'$ is an independent copy of $\xi$.
By the moment inequalities for decoupled $U$-statistics \cite[Proposition 2.4]{gine2000}, we have
\begin{equation}
\label{eq:moment-cross}
\Expect\abs{\xi^\top Q \xi'}^p
\le C_p \pth{\sigma^{2p} \Fnorm{Q}^p  + \sigma^p M_p \|Q\|_{2,p}^p+M_p^2\| Q\|_{p,p}^p}
\le 3C_p\Fnorm{Q}^pM_p^2,
\end{equation}
where $\| \cdot \|_{p,q}$ denotes the $L_{p,q}$-norm given by $\| A \|_{p,q}^q=\sum_{j}(\sum_{i} |A_{ij}|^p)^{q/p}$.
Finally, we apply \eqref{eq:moment-diag} -- \eqref{eq:moment-cross} in the upper bound \eqref{eq:tail-moment} and obtain 
\begin{align}
\Prob\qth{\abs{\xi^\top Q \xi - \Expect[\xi^\top Q \xi]}>\delta}
 \le C'_p\pth{M_p\pth{\frac{\Fnorm{Q}}{\delta}}^{p/2}+M_p^2\pth{\frac{\Fnorm{Q}}{\delta}}^{p}}
 \label{eq:moment_tail_bound_1}
\end{align}
for some constant $C'_p \ge 1$ only depends on $p.$
We claim that 
\begin{align*}
\Prob\qth{\abs{\xi^\top Q \xi - \Expect[\xi^\top Q \xi]}>\delta}
 \le 2 C'_p M_p\pth{\frac{\Fnorm{Q}}{\delta}}^{p/2}.
\end{align*}
This is because when $\Fnorm{Q} \ge \delta$, the last display equation automatically holds; otherwise, it follows from~\prettyref{eq:moment_tail_bound_1}.

Recalling $\delta=N \kappa/(e\eta ts)$ and 
$
\fnorm{Q}
\le \sqrt{\norm{Q}}\sqrt{\Tr(Q)} \le \kappa 
\sqrt{N\eta ts} \calR(1/\sqrt{\eta ts}),
$
we deduce that for $t \le T,$
\begin{align*}
q & \le 
2C'_p M_p \left( \frac{ \kappa \sqrt{N\eta ts} \calR(1/\sqrt{\eta ts})  e \eta ts }{N \kappa } \right)^{p/2} \\
& \le 2C'_p M_p  \left( \frac{  e \eta ts }{2 N\sigma^2 } \right)^{p/4}.\qedhere
\end{align*}
\end{proof}

Finally we deduce the exponential convergence result from \prettyref{prop:pred_error_general} 
in the special case of finite-rank kernels. 
\begin{proof}[Proof of \prettyref{thm: exponential}]
Since $\Lambda_{i}=0$ for $i>d$, in view of~\prettyref{lmm:bias}, we have
$$
\delta_1(t) \le \frac{1}{s} \max_{1 \le i \le d} \left( 1- \eta \Lambda_{i} \right)^{2t}\Lambda_{i}
\le \frac{1}{\eta s} \left( 1- \eta \lambda_d s/\kappa \right)^{2t}
 \le \frac{1}{\eta s} \exp \left( - 2  \lambda_d \eta t s /\kappa \right),
$$
where the second inequality holds due to $\Lambda_{d} \ge \lambda_d s/\kappa$ in view of~\prettyref{eq:eigen_comp}. 
Moreover, in view of~\prettyref{lmm:variance}, we have
$
\delta_2(t) \le d/N.
$
It follows from~\prettyref{prop:pred_error_general} that
\[
\frac{1}{N} \expects{\norm{f_t(\bx)- f(\bx)}^2}{\xi}
\le 3 \frac{ \kappa }{\eta s}  \lnorm{f_0-f}{\calH}^2\exp\left( - 2 \frac{\eta s}{\kappa} \lambda_d t \right) + 
3 \kappa \sigma^2 \frac{d}{N}  + \frac{3\kappa}{N} \norm{\Delta_{f}}^2, \quad \forall t. \qedhere
\]
\end{proof}

\subsection{Upper bound to the RKHS norm}

\begin{lemma}\label{lmm:bound_ft_H}
There exists a universal constant $c$
such that, for any $t\le T$, 
with probability at least
$1-\exp \left( -c N/(\sigma^2 \eta t s) \right)$, 
\[
\Hnorm{f_t-f} \le \Hnorm{f_0-f} + 1 +  
\sqrt{\frac{\eta s t}{N}}\norm{\Delta_{f}},
\quad \forall~f\in\calH. 
\]
\end{lemma}
\begin{proof}
Similar to \eqref{eq:L-noise}, for any $f\in\calH$, we use $\Delta_f$ in \eqref{eq:def-Delta-f} and obtain
\begin{equation}
\label{eq:ft-f-decompose}
f_t - f = \calL^t(f_0-f)+\sum_{\tau=0}^{t-1}\calL^{\tau}( (\xi+\Delta_f)\cdot \Psi).
\end{equation}
It follows from \prettyref{lmm:eigen-KP} that $\opnorm{\calL}\le 1$ and 
\begin{equation}
\label{eq:hnorm-f0-f}
\Hnorm{\calL^t(f_0-f)}\le \Hnorm{f_0-f}.
\end{equation}
For the second term of \eqref{eq:ft-f-decompose}, using the matrix $\Sigma$ defined in \eqref{eq:def-matrix-T},
we have 
$\Hnorm{\sum_{\tau=0}^{t-1}\calL^{\tau}(a\cdot \Psi)}^2=a^\top \Sigma a$ for any $a\in\reals^N$.
Applying \prettyref{lmm:norm-trace} with $\calT=\sum_{\tau=0}^{t-1}\calL^{\tau}$ yields that
\[
\norm{\Sigma}
\le \frac{s \eta}{N}\opnorm{\sum_{\tau=0}^{t-1}\calL^{\tau}(\calI-\calL^t)}
\le \frac{s \eta t}{N}.
\]
Therefore,
\begin{equation}
\label{eq:hnorm-delta-f}
\hnorm{\sum_{\tau=0}^{t-1}\calL^{\tau}( \Delta_f \cdot \Psi)}
=\sqrt{\Delta_f \Sigma \Delta_f}
\le \sqrt{\norm{\Sigma}}\norm{\Delta_f}
\le \sqrt{\frac{\eta s t}{N}}\norm{\Delta_{f}}.
\end{equation}
Finally we consider $\Hnorm{\sum_{\tau=0}^{t-1}\calL^{\tau}( \xi \cdot \Psi)}^2 = \xi^\top \Sigma \xi$.
Recall the early stopping rule~\prettyref{eq:def_early_stopping}, which implies that $\eta ts \calR(1/\sqrt{\eta ts}) \le 1/(\sqrt{2e} \sigma)$ for $t\le T$.
Then, by \prettyref{lmm:T-property},
\begin{equation}
\label{eq:theta-tr-Q-expect-ub}
\Expect[\xi^\top \Sigma \xi]
=\Expect[\Tr(\xi\xi^\top \Sigma)]
\le \sigma^2 \Tr(\Sigma)
\le \pth{\sigma \eta t s \calR \left( \frac{1}{\sqrt{\eta ts} }\right)}^2
\le \frac{1}{2e}.
\end{equation}
Using the Hanson-Wright inequality~\cite{rudelson2013hanson}, for a universal constant $c_1$,
$$
\prob{ \xi^\top \Sigma \xi - \Expect[\xi^\top \Sigma \xi]\ge \delta}
\le \exp \left( -c_1 \min \left\{ \frac{\delta}{\sigma^2 \norm{\Sigma}}, \frac{\delta^2}{\sigma^4 \fnorm{\Sigma}^2} \right\} \right).
$$
Since $ \fnorm{\Sigma}^2 \le \norm{\Sigma} \Tr(\Sigma)$.
Choosing $\delta = \frac{1}{2e}$ and invoking
$\sigma^2 \Tr(\Sigma) \le \delta$ from \eqref{eq:theta-tr-Q-expect-ub} and $\norm{\Sigma}\le \eta s t/N$, we get that 
\begin{equation}
\label{eq:theta-hp-deviation}
\prob{ \xi^\top \Sigma \xi - \Expect[\xi^\top \Sigma \xi]\ge \delta}
\le \exp \left( -c_1 N /(2\sigma^2 e \eta s t) \right).
\end{equation}
Hence, 
combining \eqref{eq:hnorm-f0-f} -- \eqref{eq:theta-hp-deviation}, we conclude the proof from \eqref{eq:ft-f-decompose}. 
\end{proof}


\begin{lemma}
\label{lmm:norm-trace}
Suppose $\calT:\calH\mapsto \calH$ is a self-adjoint linear operator. 
Let $A$ be a $N\times N$ matrix with $A_{ij}=\Hprod{\calT(\Psi_i)}{\calT(\Psi_j)}$.
Then,
\[
\norm{A}\le \frac{s\eta}{N}\opnorm{\calT(\calI-\calL)\calT},
\qquad
\Tr(A) \le \frac{s\eta}{N}\Tr(\calT(\calI-\calL)\calT).
\]
\end{lemma}
\begin{proof}
By definition, $a^\top A a = \hnorm{\calT(a\cdot \Psi)}^2$ for any $a \in \reals^N$. Therefore, 
\[
\norm{A}
=\max_{\norm{a}\le 1}\hnorm{\calT(a\cdot \Psi)}^2
=\max_{\norm{a}\le 1}\max_{\hnorm{f}\le 1}\hprod{f}{\calT(a\cdot \Psi)}^2.
\]
For any $f\in\calH$ and $a\in\reals^N$ with $\hnorm{f}\le 1$ and $\norm{a}\le 1$, we have
\[
\hprod{f}{\calT(a\cdot \Psi)}
=\hprod{\calT f}{a \cdot \Psi}
=\frac{\eta}{N}a^\top P g(\bx)
\le \frac{\eta}{N}\norm{P g(\bx)},
\]
where $g=\calT f$ and the second equality used \prettyref{lmm:block_magic}.
Using $\Norm{P}\le s$ and \prettyref{lmm:block_magic}, we get
\[
\hprod{f}{\calT(a \cdot \Psi)}^2
\le \frac{s\eta^2}{N^2}g(\bx)^\top P g(\bx)
=\frac{s\eta}{N}\Hprod{g}{g-\calL g}
=\frac{s\eta}{N}\Hprod{f}{\calT (\calI-\calL) \calT f}.
\]

Next we prove the second inequality. 
Let $\{\phi_1,\phi_2,\dots\}$ be an orthonormal basis of $\calH$, and let $f_i\triangleq \calT \phi_i$.
By the definition of $\calT$ and \prettyref{lmm:block_magic},
\begin{align*}
\Tr(A)
&=\sum_j\hnorm{\calT(\Psi_j)}^2
=\sum_{ij}\hprod{\phi_i}{\calT (\Psi_j)}^2
=\sum_{ij}\hprod{f_i}{e_j \cdot \Psi}^2\\
&=\sum_{ij}\pth{\frac{\eta}{N}e_j^\top P f_i(\bx)}^2
=\sum_{i}\norm{\frac{\eta}{N}P f_i(\bx)}^2.
\end{align*}
Since $\Norm{P}\le s$, we further have 
\[
\Tr(A)
\le \frac{s\eta^2}{N^2}\sum_{i}f_i(\bx)^\top Pf_i(\bx)
= \frac{s\eta}{N}\sum_{i}\hprod{f_i}{f_i-\calL f_i}
= \frac{s\eta}{N}\sum_{i}\hprod{\phi_i}{\calT(I-\calL)\calT \phi_i}.\qedhere
\]
\end{proof}

\section{Proofs in Section \ref{sec:theta_err} }

Again we focus on proving the results for \FedAvg. 
The proof for \FedProx follows verbatim using the facts that $\norm{P} \le 1$ and that $\norm{P^{-1}} \le \kappa$.  

\subsection{Proof of~\prettyref{thm:conv_theta}}
\label{sec:proof_conv_theta}
Since the desired conclusion~\prettyref{eq:theta_bound_desired} trivially holds when $\rho_N=0$, we assume $\rho_N>0$ in the  proof.

It follows from~\prettyref{prop:theta_recursion} and \eqref{eq:fa-equality} that
\begin{align}
f_t - \fa 
&= \calL (f_{t-1}-\fa) - (\fa -\calL \fa) + y \cdot \Psi
=\calL (f_{t-1}-\fa) + \xi\cdot \Psi\nonumber\\
&= \calL^t (f_{0}-\fa) + \sum_{\tau=0}^{t-1}\calL^\tau (\xi\cdot \Psi).\label{eq:L-noise}
\end{align}
To analyze \eqref{eq:L-noise}, we show properties of $\calL$ and the matrix $\Sigma$ of size $N\times N$ with 
\begin{equation}
\label{eq:def-matrix-T}
\Sigma_{ij}=\hprod{\sum_{\tau=0}^{t-1}\calL^\tau (e_i\cdot \Psi)}{\sum_{\tau=0}^{t-1}\calL^\tau (e_j\cdot \Psi)}.
\end{equation}

\begin{lemma}
\label{lmm:L-psd}
For $f \in \calH$, define $\calP f=\frac{1}{N}f(\bx)\cdot k_{\bx}$.
Then,
\begin{align}
\calI - s \eta \calP \preceq \calL\preceq \calI - s \eta \calP / \kappa, \label{eq:P_operator}
\end{align}
where $\calT_1 \preceq \calT_2$ means $\calT_2-\calT_1$ is positive.

Moreover, assume $\phi(x)$ is $d$-dimensional. Then there is a one-to-one correspondence between the eigenvalues of $\calP$ and those of the $d$ by $d$ matrix $\frac{1}{N}\phi(\bx)^\top \phi(\bx)$.
\end{lemma}
\begin{proof}
We first show that
\begin{equation}
\label{eq:Li-psd}
\calI - s(\calI-\calL_i)
\preceq \calL_i^s = (\calI - (\calI-\calL_i))^s 
\preceq \calI - s(\calI-\calL_i)/\kappa.
\end{equation}
Since all terms in \eqref{eq:Li-psd} are polynomial in $\calL_i$, it suffices to show the ordering of corresponding eigenvalues. 
Suppose $\lambda_j$ is the $j$-th eigenvalue of $\calI-\calL_i$.
It is shown in the proof of \prettyref{lmm:eigen-KP} that $0 \le \lambda_j \le \gamma$.
Then,
\[
1-s\lambda_j \le (1-\lambda_j)^s \le 1-s\lambda_j/\kappa.
\]
To see the second inequality, we note the function $ x\mapsto \frac{x}{1-(1-x)^s}$ is monotone increasing in $[0,1]$ and thus
\[
\kappa = \frac{\gamma s}{1-(1-\gamma)^s}\ge \frac{\lambda_j s}{1-(1-\lambda_j)^s}.
\]
Then~\prettyref{eq:P_operator} follows from \eqref{eq:Li-psd} as $\calL=\sum_i w_i \calL^s_i$.

It remains to establish the correspondence between the eigenvalues of $\calP$ and those of $\frac{1}{N}\phi(\bx)^\top \phi(\bx)$. 
Recall that $\{\phi_\ell\}_{\ell=1}^d$ forms an orthonormal basis of $\calH$.
 Thus, it suffices to show a matrix representation of $\calP$ is $\frac{1}{N}\phi(\bx)^\top \phi(\bx) \phi$, \ie, 
 for any $f=a \cdot \phi$ with $a \in \reals^d$, 
we have $\calP f= \left( \frac{1}{N}\phi(\bx)^\top \phi(\bx) a \right)\cdot \phi$.
This follows from the fact that 
$\calP \phi = \frac{1}{N}\phi(\bx)^\top \phi(\bx) \phi$ for 
 $\phi=\left(\phi_1,\ldots, \phi_d\right)$.

\end{proof}

\begin{lemma}
\label{lmm:T-property}
Let $\{\tilde\lambda_1,\tilde\lambda_2,\dots\}$ be the eigenvalues of $\calL$. Then,
\[
\Tr(\Sigma) 
\le \frac{\eta s}{N}\sum_{i} \frac{ (1-\tilde\lambda_i^t)^2 }{1-\tilde\lambda_i}
\le (\eta t s)^2 \calR^2\pth{\frac{1}{\sqrt{\eta t s}}},
\]
where $\calR$ is the empirical Rademacher complexity defined in \eqref{eq:R_K_def}.
\end{lemma}
\begin{proof}
Applying \prettyref{lmm:norm-trace} with $\calT=\sum_{\tau=0}^{t-1}\calL^\tau$ yields that
\[
\Tr(\Sigma)
\le \frac{s \eta}{N}\Tr\pth{\sum_{\tau=0}^{t-1}\calL^\tau (I-\calL^t)}.
\]
Let $\{\tilde \lambda_1,\tilde \lambda_2,\dots\}$ denote the eigenvalues of $\calL$, where $\tilde \lambda_i\in [0,1]$ by \prettyref{lmm:condition-P}.
Applying the facts $1-x^t\le \min\{1,t(1-x)\}$ and $\min\{\frac{1}{x},t^2x\} \le \min\{t,t^2x\}$ for $t\ge 0$ and $0\le x\le 1$, we obtain
\begin{align*}
\Tr\left(\sum_{\tau=0}^{t-1}\calL^\tau (I-\calL^t) \right) 
& =\sum_{i} \frac{ (1-\tilde\lambda_i^t)^2 }{1-\tilde\lambda_i}  \\
& \le \sum_{i} \min \left\{ \frac{1}{1-\tilde\lambda_i}, t^2 (1-\tilde\lambda_i)\right\}  \\
& \le  \sum_{i} \min \left\{ t, t^2 (1-\tilde\lambda_i)\right\}. 
\end{align*}
By \prettyref{lmm:L-psd}, 
we have $1-\tilde\lambda_i \le s \eta \lambda_i$
for $1 \le i \le N$ and $1-\tilde\lambda_i=0$ for $i > N$.
It follows that 
\[
\Tr(\Sigma)
\le \frac{s\eta }{N}\sum_{i=1}^N \min \left\{ t, t^2 s \eta \lambda_i \right\}
=t^2s^2\eta^2 \calR^2\left( \frac{1}{\sqrt{\eta ts} }\right),
\]
where the last equality used the definition of $\calR$ in \eqref{eq:R_K_def}.
\end{proof}

For the first term of \eqref{eq:L-noise},
by \prettyref{lmm:L-psd}, we get
\begin{equation}
\label{eq:theta-error-contraction}
\opnorm{\calL} \le 1-\frac{s\eta \lambda_{\min}(\calP)}{\kappa}
=1-\frac{s\eta \rho_N}{\kappa}.
\end{equation}
By linearity, the norm of the second term in \eqref{eq:L-noise} can be represented as 
\[
\hnorm{\sum_{\tau=0}^{t-1}\calL^\tau (\xi\cdot \Psi)}^2
=\xi^\top \Sigma \xi =\Tr(\xi\xi^\top \Sigma).
\]
By the assumption $\expect{\xi \xi^\top} \preceq \sigma^2 I$ and \eqref{eq:trace_fact}, we have
\[
\Expect \left[ \Tr(\xi\xi^\top \Sigma) \right]
\le \sigma^2 \Tr(\Sigma)
\le \sigma^2\frac{\eta s }{N}\sum_i \frac{ (1-\tilde\lambda_i^t)^2 }{1-\tilde\lambda_i} .
\]
Recall that $s\eta \rho_N /\kappa \le 1-\tilde \lambda_i \le 1$ by \prettyref{lmm:eigen-KP} and \eqref{eq:theta-error-contraction}, and that $\phi$ is $d$-dimensional.
We obtain
\begin{equation}
\label{eq:theta-error-noise}
\Expect\Tr(\xi\xi^\top \Sigma)
\le \sigma^2\frac{\eta s }{N}\sum_{i=1}^d\frac{1}{s\eta \rho_N/\kappa}
= \frac{\sigma^2 \kappa d}{N \rho_N}.
\end{equation}
Applying \eqref{eq:theta-error-contraction} and \eqref{eq:theta-error-noise} to \eqref{eq:L-noise} yields the desired~\prettyref{eq:theta_bound_desired}.

It remains to establish~\prettyref{eq:model_bound}.  
By the definition of $\fa$, 
\[
\fa -f_j^* = (\calI-\calL)^{-1}\pth{ \Delta_{f_j^*} \cdot \Psi}.
\]
where $\Delta_{f_j^*} = ( f_1^*(\bx_1)-f_j^*(\bx_1),\dots,f_M^*(\bx_M)-f_j^*(\bx_M))$ defined in \eqref{eq:def-Delta-f}.
Then by linearity,
\[
\Hnorm{\fa -f_j^*}^2
= \Delta_{f_j^*}^\top S \Delta_{f_j^*}
\le \norm{\Delta_{f_j^*}}^2 \norm{S},
\]
where $S$ is a matrix of size $N \times N$ with
$
S_{ij}=\hprod{(\calI-\calL)^{-1} (e_i\cdot \Psi)}{(\calI-\calL)^{-1} (e_j\cdot \Psi)}.
$
Applying \prettyref{lmm:norm-trace} with $\calT=(\calI-\calL)^{-1}$ yields that
\[
\norm{S}\le \frac{s\eta}{N}\opnorm{(\calI-\calL)^{-1}}.
\]
It follows from \eqref{eq:theta-error-contraction} that $\norm{S}\le \frac{\kappa}{N\rho_N}$, which implies \prettyref{eq:model_bound}.

\subsection{Proofs of Corollaries \ref{cor:conv_theta} -- \ref{cor:subspace}}

\begin{proof}[Proof of \prettyref{cor:conv_theta}]
In view of 
\cite[Theorem 5.39]{vershynin2010nonasym} and the union bound, with probability at least $1-e^{-d}$,
$$
\norm{ \frac{1}{N} \phi(\bx)^\top \phi(\bx)
- \frac{1}{N}\sum_{i=1}^M\sum_{j=1}^{n_i}\Sigma_{ij} } 
\le c_1 \max\left\{ \sqrt{ \frac{d}{N} } , \frac{d}{N} \right\},
$$
where $c_1$ is a universal constant.
By the assumption, $\alpha I \preceq \Sigma_{ij} \preceq \beta I$ for some fixed constant $\alpha, \beta>0.$
Therefore,  $$
\rho_N = \lambda_{\min}
\left(\frac{1}{N} \phi(\bx)^\top \phi(\bx) \right)
\ge \alpha - c_1 \max\left\{ \sqrt{ \frac{d}{N} } , \frac{d}{N} \right\}
\ge \alpha/2,
$$
where the first inequality follows from Weyl's inequality
and the second inequality holds by choosing 
$N \ge d \max\{ 4 c_1^2/\alpha^2, 2c_1/\alpha\}$.
The desired \prettyref {eq:conv_theta_random} readily follows from 
\eqref{eq:theta_bound_desired}.

It remains to prove~\eqref{eq:model_theta_random}.
By the definition of $\Norm{\Delta_{f_j^*}}$, we have
\begin{equation}
\label{eq:Delta-Gamma}
\norm{\Delta_{f_j^*}}^2
= \sum_{i=1}^M\norm{f_i^*(\bx_i)-f_j^*(\bx_i)}^2
=\sum_{i=1}^M \sum_{j=1}^{n_i}
 \iprod{\phi(x_{ij})}{\theta_i^*-\theta_j^*}^2.
\end{equation}
Recall that $\Gamma=\max_{i,j} \norm{\theta_i^*-\theta_j^*}$. Thus $\iprod{\phi(x_{ij})}{\theta_i^*-\theta_j^*}$ are independent and sub-Gaussian random variables with 
the sub-Gaussian norm bounded by $c_2\Gamma$ for a constant $c_2$.
It follows from the Hanson-Wright inequality that 
$$
\prob{\norm{\Delta_{f_j^*}}^2 \ge 
\expect{\norm{\Delta_{f_j^*}}^2}
+t }
\le \exp\left( - c_3 \min \left(
\frac{t^2}{\Gamma^4 N}, \frac{t}{\Gamma^2} \right) \right).
$$
Setting $t=c_4 \Gamma^2 N$ for some large constant $c_4$, we get that
with probability at least $1-\exp\left( - N  \right)$, 
$$
\norm{\Delta_{f_j^*}}^2
\le \expect{\norm{\Delta_{f_j^*}}^2}+ c_4 \Gamma^2 N
\le \Gamma^2 \left(\beta+c_4\right)N,
$$
where the last inequality holds due to $\Sigma_{ij} \preceq \beta I$.
The conclusion \eqref{eq:model_theta_random} follows from~\prettyref{eq:model_bound}.
\end{proof}

\begin{proof}[Proof of~\prettyref{cor:subspace}]
We first lower bound $\rho_N$.
By assumption $F_i^\top F_i/n_i  \succeq \alpha I_d$, and then
\begin{align}
\frac{1}{N} \phi(\bx)^\top \phi(\bx) 
= \frac{1}{N} \sum_{i=1}^M \phi (\bx_i)^\top \phi (\bx_i) 
= \frac{1}{N} \sum_{i=1}^M \frac{d}{r_i} U_i F_i^\top F_i U_i^\top
\succeq \alpha \sum_{i=1}^M  w_i\frac{d}{r_i}
U_i U_i^\top. 
\label{eq:subspace_1}
\end{align}
Note that 
$\expect{U_i U_i^\top} = \frac{r_i}{d} I_d$. Thus,
\begin{equation}
\label{eq:phi-cov-expectation}
\sum_{i=1}^M  w_i \frac{d}{r_i} 
\expect{U_i U_i^\top}= I_d,
\end{equation}
Let 
$$
Y_i=  w_i \frac{d}{r_i}\left[ U_i U_i^\top - \expect{U_i U_i^\top}\right].
$$
Next we use the matrix Bernstein  inequality to bound
the deviation $\sum_{i=1}^M Y_i $. 
Note that $\opnorm{Y_i} \le 2 w_i d/r_i$
and
\begin{align*}
\opnorm{\sum_{i=1}^M \expect{Y_i^2} }
&=\opnorm{ \sum_{i=1}^M w_i^2 \frac{d^2}{r_i^2}\left( \expect{U_i U_i^\top} - \left( \expect{U_i U_i^\top} \right)^2 \right) } \\
&=\opnorm{\sum_{i=1}^M  w_i^2 \frac{d^2}{r_i^2} \left( \frac{r_i}{d} I_d -  \left(\frac{r_i}{d}  \right)^2 I_d \right)  }\\
&\le \sum_{i=1}^M w_i^2 \frac{d}{r_i}.
\end{align*}
Therefore, by  the matrix Bernstein  inequality, with probability at least $1-d^{-1}$, for a universal constant $c_3>0$,
\begin{align}
\opnorm{\sum_{i=1}^M Y_i } & \le c_3 \sqrt{\sum_{i=1}^M w_i^2  \frac{d}{r_i} \log d } + c_3 
\max_{1 \le i \le M} 
w_i \frac{d}{r_i} \log d \nonumber\\
& \overset{(a)}{\le} 
c_3 \sqrt{ \frac{\nu d\log d}{N} } + c_3 
 \frac{ \nu d \log d}{N} \nonumber\\
& \overset{(b)}{\le} \frac{1}{2},\label{eq:phi-cov-perturb}
\end{align}
where $(a)$ holds 
by definition 
$\nu = \max_{i \in [M]} n_i/r_i$ and $w_i=n_i/N$; $(b)$ holds by the assumption that $N \ge C \nu d \log d$ for a sufficiently large constant $C$.
Therefore, combining \eqref{eq:phi-cov-expectation} and \eqref{eq:phi-cov-perturb},
$$
\frac{1}{N} \phi(\bx)^\top \phi(\bx)
\succeq  \alpha\sum_{i=1}^M w_i \frac{d}{r_i} U_i U_i^\top \succeq \frac{\alpha}{2} I_d.
$$
Thus the desired conclusion \prettyref {eq:conv_theta_random_2} readily follows from \eqref{eq:theta_bound_desired} in~\prettyref{thm:conv_theta}.
The proof of~\prettyref{eq:model_theta_random_2} follows 
similarly from \eqref{eq:Delta-Gamma} and
\[
\sum_{i=1}^M\norm{f_i^*(\bx_i)-f_j^*(\bx_i)}^2
\le 
\sum_{i=1}^M \norm{\phi(\bx_i)}^2\norm{\theta_i^*-\theta_j^*}^2,
\]
\[
\sum_{i=1}^M \norm{\phi(\bx_i)}^2 
\le \sum_{i=1}^M \frac{d}{r_i}\norm{F_i}^2 
\le \beta  d\sum_{i=1}^M \frac{n_i }{r_i}
\le \beta \nu Md. \qedhere
\]
\end{proof}

\section{Proofs in Section~\ref{subsec: Characterization of Federation gains}}

\begin{proof}[Proof of~\prettyref{thm:FG}]
It follows from~\prettyref{cor:conv_theta} that 
$$
R^{\mathsf{Fed}}_j \lesssim \kappa  \left(\sigma^2 d/N + \Gamma^2  \right). 
$$
Then the desired conclusion readily follows from the following claim:
\[
 R^{\mathsf{Loc}}_j \gtrsim
  \min\{ \sigma^2 d/n_j, B^2\} + \max\{1- n_j/d, 0\} B^2.
\]
It remains to check the claim. Note that to estimate the model $f_j^*\in \calH_B$, it is equivalent to estimating the model coefficient $\theta_j^*$ in the $\ell^2$ ball of radius $B$ centered at the origin.

For any $n_j$ and $d$,
we bound the minimax risk from below using the celebrated Fano's inequality. Let $\calV$
denote a $1/2$-packing set of the unit $\ell^2$ ball in $\ell^2$ norm. By simple volume ratio argument (see e.g.~\cite[Lemma 5.5 and Lemma 5.6]{wainwright2019high}, such a set of cardinarlity $|\calV|\ge 2^d$ exists.
For each $v \in \calV$, define $\theta_v=4\delta v$, where 
$\delta$ will be optimized later. For every pair of $v \neq v'$, $\norm{\theta_v-\theta_{v'}} = 4\delta \norm{v-v'} \ge 2\delta$. Also, let $\calP_v$ denote the distribution of $y_j$ conditional on $\theta_j^*=\theta_v$. 
Let $D_{\rm KL}$ denote the Kullback–Leibler divergence.
Then 
by the convexity of $D_{\rm KL}$, 
we have
\begin{align*}
D_{\rm KL} \left( \calP_v \| \calP_{v'}\right)
& \le \expects{D_{\rm KL} \left( \calN\left(\phi(\bx_j) \theta_v, \sigma^2 \identity \right)\| \calN \left(\phi(\bx_j) \theta_v', \sigma^2 \identity \right) \right)}{\bx_j} \\
& =\expects{\frac{\norm{\phi(\bx_j) (\theta_v-\theta_v')}^2}{2\sigma^2}}{\bx_j} \\
& \le n_j \beta \norm{\theta_v-\theta_v'}^2 /(2\sigma^2) \\
& \le n_j \beta 32 \delta^2/\sigma^2.
\end{align*}
Therefore, 
$$
\max_{v, v'} 
D_{\rm KL} \left( \calP_v \| \calP_{v'}\right)
\le 32 n_j \beta  \delta^2/\sigma^2.
$$
Finally, applying Fano's inequality (see e.g.~\cite[Proposition 15.12]{wainwright2019high}), we get 
$$
R^{\mathsf{Loc}}_j \ge \delta^2 \left( 1 - \frac{32 n_j \beta  \delta^2/\sigma^2 + \log 2}{d \log 2} \right).
$$
Picking $\delta^2=\frac{1}{64} \min\{ \sigma^2 d /(\beta n_j), B^2\}$. Then by construction, $\norm{\theta_v} \le B$. Further, it follows from the last displayed equation that 
$$
R^{\mathsf{Loc}}_j \ge c(\beta) \min\{ \sigma^2 d /(\beta n_j), B^2\},
$$
where $c(\beta)$ is a constant that only depends on $\beta.$

When $n_j<d$, we bound the minimax risk from below by assuming $\theta_j^*$
is uniformly distributed over the $\ell^2$ sphere $\calS$ of radius $B$. Moreover, we use the standard genie-aided argument by assuming that the estimator also has access to $\bar{y}_j\triangleq \phi(\bx_j)\theta_j^*$. In this case, the posterior distribution of $\theta_j^*$ (conditional on $\{\bx_j, y_j, \bar{y}_j\}$) is the uniform distribution over $
\calS'\triangleq \calS \cap \{\theta: \phi(\bx_j) \theta =\bar{y}_j\}$. Construct matrix $V$ (resp.\ $V_{\perp}$) by choosing
its columns as a set of basis vectors in the row (resp.\ null) space of $\phi(\bx_j).$ Then 
$\calS'=\{\theta: \theta=V_{\perp} \alpha + V \beta \}$, where $\beta$ is the unique solution such that 
$\phi(\bx_j)V\beta=\bar{y}_j$ and $\alpha$ satisfies $\norm{\alpha}^2=B^2-\norm{\beta}^2.$
Therefore,  for any estimator $\hat{\theta}_j \left( \bx_j, y_j, \bar{y}_j \right)$,
$$
\expect{\norm{\hat{\theta}_j - \theta_j^* }^2 \mid \bx_j, y_j, \bar{y}_j }
\ge \inf_{\theta} 
\expect{ \norm{\theta-\theta_j^*}^2 \mid \bx_j, y_j, \bar{y}_j } = 
B^2-\norm{\beta}^2
$$
Taking the average over both hand sides and using 
$V^\top \theta_j^*=\beta$ so that 
$$
\expects{\norm{\hat{\theta}_j - \theta_j^*}^2}{\theta_j^*,\bx_j,\xi_j}
\ge B^2 - \expect{\norm{V^\top \theta_j^*}^2}
\ge \left( 1-\frac{n_j}{d}\right) B^2, \quad \forall \hat{\theta}_j,
$$
where the last inequality holds because the rank of $V$ is at most $n_j$ and the prior distribution of $\theta_j^*$
is uniform over the sphere $\calS$, so that 
$$
\expect{\norm{V^\top \theta_j^*}^2 \mid \bx_j}=
\iprod{VV^\top}{\expect{\theta_j^* (\theta_j^*)^\top}}
=\frac{B^2}{d} \iprod{VV^\top}{\identity}
=\frac{B^2}{d} \fnorm{V}^2 \le \frac{B^2 n_j}{d}.
$$
Therefore, 
\[
R^{\mathsf{Loc}}_j  \ge \inf_{ \hat{\theta}_j } \expects{\norm{\hat{\theta}_j - \theta_j^*}^2}{\theta_j^*,\bx_j,\xi_j} ] \ge \left( 1-\frac{n_j}{d}\right) B^2.
\qedhere
\]
\end{proof}

\begin{proof}[Proof of~\prettyref{thm:fg_subspace}]
It follows from~\prettyref{eq:conv_theta_random_2} in~\prettyref{cor:subspace} that 
$$
R^{\mathsf{Fed}}_j \lesssim \sigma^2 \kappa d/N.
$$
Then the desired conclusion readily follows from the following claim:
\[
 R^{\mathsf{Loc}}_j \gtrsim
 \min\{ \sigma^2 d/n_j, B^2\} + (1- r_j/d) B^2.
\]
The proof of the claim follows verbatim as that in~\prettyref{thm:FG} with the rank of $V$ being at most $r_j$, and is omitted for simplicity. 
%
%
\end{proof}

\end{document}